\documentclass[conference]{IEEEtran}
\usepackage{cite}
\usepackage{amsmath,amssymb,amsfonts,amsthm}
\usepackage{algorithmic}
\usepackage{graphicx}
\usepackage{textcomp}
\usepackage{xcolor}
\usepackage{float}
\usepackage[unicode]{hyperref}
\usepackage{mathtools}
\usepackage{algorithm,algorithmic}

\theoremstyle{plain}
\newtheorem{theorem}{Theorem}
\newtheorem{proposition}[theorem]{Proposition}

\theoremstyle{definition}
\newtheorem{definition}[theorem]{Definition}

\allowdisplaybreaks
\newcommand{\argmin}{\operatorname*{argmin}}

\newcommand{\diag}{\operatorname{diag}}
\newcommand{\tr}{\operatorname{tr}}
\newcommand{\St}{\operatorname{St}}
\newcommand{\grad}{\operatorname{grad}}
\newcommand{\symm}{\operatorname{symm}}
\newcommand{\proj}{\mathrm{Proj}}
\newcommand{\retr}{\mathrm{Retr}}
\newcommand{\uf}{\operatorname{uf}}

\newcommand{\med}{\operatorname{med}}
\newcommand{\email}[1]{\href{mailto:#1}{#1}}

\usepackage{tikz}
\usetikzlibrary{shapes,calc}
\usepackage{booktabs,multicol,multirow}

\def\BibTeX{{\rm B\kern-.05em{\sc i\kern-.025em b}\kern-.08em
    T\kern-.1667em\lower.7ex\hbox{E}\kern-.125emX}}
\begin{document}

\title{Geometry-Aware Bayesian Parameter-Efficient Fine-Tuning on the Stiefel Manifold via Stein~Variational~Gradient~Descent}

\author{
    \IEEEauthorblockN{Quang-Duy Tran\IEEEauthorrefmark{1}\IEEEauthorrefmark{2}\IEEEauthorrefmark{4}, Trung Le\IEEEauthorrefmark{3}, Bao Duong\IEEEauthorrefmark{1}, Phuoc Nguyen\IEEEauthorrefmark{1}, Thin Nguyen\IEEEauthorrefmark{1}}
    \IEEEauthorblockA{
        \IEEEauthorrefmark{1}Deakin Applied Artificial Intelligence Initiative, Deakin University, Geelong, Australia\\
        \IEEEauthorrefmark{2}Department of Computer Science, Aalto University, Espoo, Finland\\
        \IEEEauthorrefmark{3}Department of Data Science \& AI, Monash University, Melbourne, Australia\\
        \email{quang-duy.tran@aalto.fi}, \email{trunglm@monash.edu}, \{\href{mailto:b.duong@deakin.edu.au}{b.duong}, \href{mailto:phuoc.nguyen@deakin.edu.au}{phuoc.nguyen}, \href{mailto:thin.nguyen@deakin.edu.au}{thin.nguyen}\}@deakin.edu.au\\
        {\IEEEauthorrefmark{4}This work was completed at the Deakin Applied Artificial Intelligence Initiative, Deakin University.}
    }
}

\maketitle

\begin{abstract}
    Several geometry-aware approaches to low-rank adaptation have emerged for parameter-efficient fine-tuning of large pre-trained models. These methods aim to take full advantage of the geometric structure of low-rank manifolds for improving the efficiency in subspace utilization and reducing redundancy by enforcing orthogonality constraints during optimization. The strong empirical results of these techniques have motivated further study into whether predictions from such geometry-based adaptation methods could be overconfident. In this paper, we build on the singular value decomposition factorization of adapters to develop a framework based on Stein variational gradient descent (SVGD). In this formulation, the low-rank matrices are transported along the Stiefel manifold to match the targeted distributions while retaining their crucial geometric structure. Since this geometry-aware SVGD approach provides multiple solutions during inference, it supports uncertainty quantification and produces better-calibrated adapters on the Stiefel manifold. Extensive experiments show that our method delivers strong model calibration and attains higher prediction accuracy than SVGD and related uncertainty estimation methods that are formulated in Euclidean space.
\end{abstract}


\section{Introduction}
Various large pre-trained models, such as large language models (LLMs)~\cite{Achiam_23GPT4,Touvron_etal_23Llama2,Grattafiori_etal24Llama3}, are emerging as valuable tools for safety-critical domains~\cite{Clusmann_etal_23Future,Zhang_etal_24Comprehensive}. To efficiently adapt these pre-trained models to a specific domain, parameter-efficient fine-tuning (PEFT) formulations are often utilized to create general instruction-following models~\cite{Hu_etal_22LoRA,Wang_23SelfIntruct,Li_etal_25StelLA,Park_etal_25Riemannian}. Low-rank adaptation (LoRA)~\cite{Hu_etal_22LoRA} treats the fine-tuned weights $\Delta \mathbf{W}$ as an additive adapter to the pre-trained weights $\mathbf{W}_{0} \in \mathbb{R}^{m\times n}$ as $\mathbf{W}\coloneqq \mathbf{W}_{0} + \Delta\mathbf{W}$ and reduces the computational cost with a low-rank decomposition $\Delta \mathbf{W} \coloneqq  \frac{\alpha}{r}\mathbf{B}\mathbf{A}$, where $\mathbf{B} \in \mathbb{R}^{m\times r}$, $\mathbf{A} \in \mathbb{R}^{r\times n}$, $r$ is the rank such that $r \ll \min\left(m, n\right)$, and $\alpha$ is a scale factor.

However, despite showing improved accuracy, it has been shown that fine-tuning can break the calibration of the models and cause them to produce overconfident predictions~\cite{Jiang_etal_21Can,Kadavath_etal_22Language,Achiam_23GPT4}. Bayesian deep learning~\cite{Blundell_etal_15Weight,Gal_Ghahramani_16Dropout,Lakshminarayanan_etal_17Simple} has been a principled choice for estimating the epistemic uncertainty as well as allowing the production of better calibrated predictions through Bayesian posterior predictive distributions of these models. Although these methods are often challenging in large models due to the high dimensionality of the parameters, applying them to a significantly reduced subset of parameters during PEFT is more feasible. In addition, due to a lower amount of training data during, fine-tuning is more prone to epistemic uncertainty compared to pre-training~\cite{Yang_etal_24Bayesian}. Hence, various Bayesian inference studies have been presented~\cite{Yang_etal_24Bayesian,Wang_etal_24BLoB,Samplawski_etal_25Scalable} to address the uncertainty and calibration during LoRA fine-tuning.

On the other hand, to further enhance the rank efficiency of LoRA, approaches~\cite{Lion_etal_25PoLAR,Li_etal_25StelLA,Park_etal_25Riemannian} have been introduced to enforce orthogonality between the $r$ basis vectors of the up-projection matrix, as in Stiefel-LoRA~\cite{Park_etal_25Riemannian}, or both the basis vectors of up-projection and down-projection matrices, as in StelLA~\cite{Li_etal_25StelLA}. These constraints are based on the assumption that the weights belong to the Stiefel manifolds and geometry-aware optimization on Riemannian manifold~\cite{Absil_08Optimization} can straightforwardly be applied to enforce the orthogonality. With this constraint, these methods can fully exploit all subspace dimensions of the rank, leading to promising performance enhancements over diverse fine-tuning tasks~\cite{Li_etal_25StelLA,Park_etal_25Riemannian}. 

Despite promising predictive results, these methods lack evaluations on calibration and epistemic uncertainty. As a result, we present this study to examine the calibration errors of geometry-aware low-rank fine-tuning adaptations (specifically, StelLA~\cite{Li_etal_25StelLA}) and proposed a Bayesian inference framework on the Stiefel manifold for evaluating their uncertainty and improving the calibration. 
Due to the complexity of distributional definitions for Bayesian inference on the manifold, particle-based variational inference approaches, notably Stein variational gradient descent (SVGD)~\cite{Liu_Wang_16Stein,Liu_Zhu_18Riemannian}, are ideal for this setting due to their approximation flexibility compared to model-based variational inference and better iteration-effectiveness compared to Monte Carlo methods~\cite{Liu_Zhu_18Riemannian}. As a result, we choose SVGD~\cite{Liu_Wang_16Stein,Liu_Zhu_18Riemannian} as our inference engine and solve to find the optimal update function to iteratively transport the particles on the Riemannian/Stiefel manifold to the target distributions. Through evaluations on diverse experimental settings, our proposed approach---Riemannian \underline{\textbf{Ste}}in variational gradient descent for \underline{\textbf{P}}EFT on the \underline{\textbf{S}}tiefel manifold (\textbf{StePS})---has demonstrated favorable accuracy and expected calibration errors compared to related methods on both in-distribution and out-of-distribution settings of the commonsense reasoning fine-tuning benchmarks. 

\paragraph*{Contributions} The key contributions of this work can be highlighted as follows:
\begin{itemize}
    \item We address the uncertainty estimation and calibration problems in geometry-aware PEFT by performing the Bayesian inference on the Stiefel manifold, allowing explicit and principled estimation of the epistemic uncertainty and enhancing the calibration of the predictions via the Bayesian posterior predictive distribution.
    \item We find the optimal transport function with the gradient flow formulation on the Stiefel manifold, resulting in the closed-formed solution for steepest descent to iteratively transport the particles to match the target distributions on the manifold.
    \item The effectiveness of our approach is evaluated through extensive experiments on fine-tuning and inference on both in-distribution and out-of-distribution data, demonstrating beneficial improvements on predictive accuracy and model calibration results compared to related Bayesian PEFT frameworks.
\end{itemize}

\section{Related Work}
\subsection{Geometry-Aware Parameter-Efficient Fine-Tuning}
Low-rank adaptation (LoRA)~\cite{Hu_etal_22LoRA} is a foundational parameter-efficient fine-tuning method that freezes the pre-trained backbone weights and injects trainable low-rank updates ($\Delta\mathbf{W} \coloneqq \mathbf{B}\mathbf{A}$) into selected weight matrices. While LoRA is simple and effective, it treats the factors as unconstrained Euclidean parameters, which can ignore the vital geometric structure of the low-rank matrix spaces. Recent work has begun to utilize geometric structure into the low-rank adaptation. Stiefel-LoRA~\cite{Park_etal_25Riemannian} directly constrains the $\mathbf{B}$ to lie on the Stiefel manifold to enforce the orthonormal columns that span across the adapted subspace. The $\mathbf{A}$ factor remains unconstrained and is optimized conventionally.

SVD-like tri-factor low-rank formulations of the adapter ($\Delta\mathbf{W} \coloneqq\mathbf{U}\mathbf{S}\mathbf{V}^{\top}$), as adopted in, for example,~\cite{Zhang_etal_23Adaptive,Li_etal_25StelLA,Lion_etal_25PoLAR,Schotthofer_25GeoLoRA}, provide a more meaningful geometric structure, where both $\mathbf{U}$ and $\mathbf{V}$ ideally form orthonormalized bases with respect to the subspace dimensions. Different strategies have been used to enforce or preserve the orthogonality. AdaLoRA~\cite{Zhang_etal_23Adaptive} introduces additional regularization terms that ``softly'' encourages orthonormality of the factors. GeoLoRA~\cite{Schotthofer_25GeoLoRA} implicitly preserves low-rank orthogonality by updating the adapter parameters along a dynamical low-rank gradient flow, where orthonormality is maintained as an invariant of the integration scheme rather than enforced explicitly. Both PoLAR~\cite{Lion_etal_25PoLAR} and StelLA~\cite{Li_etal_25StelLA} impose explicit manifold constraints on the factors.
Specifically, PoLAR employs a landing algorithm with infeasibility penalty to guide the factors toward the Stiefel manifold, whereas StelLA performs the optimization directly on the Stiefel manifold using tangent-space projection and retraction operations. Consequently, due to its principled geometry constraints, StelLA is particularly well-suited for geometry-aware Bayesian inference.

\subsection{Uncertainty Quantification of Large Models}
Large pre-trained models tend to exhibit good calibration during pre-training~\cite{Jiang_etal_21Can,Kadavath_etal_22Language,Achiam_23GPT4}. However, they are often failed to express reliable predictive uncertainty after fine-tuning~\cite{Yang_etal_24Bayesian,Wang_etal_24BLoB,Balabanov_25Uncertainty}. To address this issue, model-based uncertainty quantification methods~\cite{Gal_Ghahramani_16Dropout,Lakshminarayanan_etal_17Simple,Balabanov_25Uncertainty,Yang_etal_24Bayesian,Wang_etal_24BLoB,Samplawski_etal_25Scalable} can be integrated to the fine-tuning adapters to allow the models to efficiently provide epistemic uncertainty estimation and improved calibration. 

While straightforward methods, such as Monte Carlo Dropout~\cite{Gal_Ghahramani_16Dropout} and Deep Ensemble~\cite{Lakshminarayanan_etal_17Simple,Balabanov_25Uncertainty}, can be directly applied, they generally yield limited enhancements in model calibration. Laplace-LoRA~\cite{Yang_etal_24Bayesian} employs Laplace approximation to LoRA parameters to evaluate the uncertainty around the Maximum A Posteriori (MAP) solution. However, its post-hoc nature can lead to suboptimal uncertainty estimation. Model-based variational inference methods, BLoB~\cite{Wang_etal_24BLoB} and ScalaBL~\cite{Samplawski_etal_25Scalable}, allow the joint optimization of the approximated variational posterior, which can enhance the efficiency and calibration. Nevertheless, their stochastic optimization procedures can be unstable and often requires additional techniques for enhancing the training stability. Despite the advantages, existing Bayesian PEFT methods are predominantly designed for Euclidean parameters, lacking the ability to capture inherent geometric structure of the low-rank parameters. This motivates the development of geometry-aware Bayesian approaches.

Beyond model-based variational inference, particle-based variational inference methods provide an alternative approach to uncertainty quantification. Stein variational gradient descent (SVGD)~\cite{Liu_Wang_16Stein} approximates the posterior distribution with a pre-defined set of samples, called ``particles'', which are deterministically transported toward the target distribution via functional gradient flows. Compared to the model-based approaches, SVGD avoids the restrictive distributional assumptions and allows more expressive approximations of the posterior distribution with a trade-off in computational cost. To account for non-Euclidean parameters, Riemannian SVGD~\cite{Liu_Zhu_18Riemannian} extends SVGD to parameters on Riemannian manifolds, enabling the inference to be performed directly on constrained parameter spaces. However, its application to large-scale deep learning models and PEFT remains unexplored. 

\section{Preliminaries}
\subsection{Tri-Factor Factorization of Adapters and Optimization on the Stiefel Manifold}\label{subsec:SVD_based_adaptation}
\paragraph{SVD-Based Adaptation} 

We follow the problem setting of StelLA~\cite{Li_etal_25StelLA} by considering a low-rank factorization based on truncated singular value decomposition (SVD) of the fine-tuning adaptation $\Delta \mathbf{W}$ for a pre-trained weight matrix $\mathbf{W}_{0} \in \mathbb{R}^{m \times n}$ as
\begin{align}
    \Delta \mathbf{W} \coloneqq \frac{\alpha}{r}\mathbf{U} \mathbf{S} \mathbf{V}^{\top},\label{eq:svd_decomp}
\end{align}
where $\mathbf{U} \in \mathbb{R}^{m \times r}$, $\mathbf{V} \in \mathbb{R}^{n \times r}$, $\mathbf{S} \coloneqq \diag\left(\mathbf{s}\right), \mathbf{s} \in \mathbb{R}^{r}$, $r \ll \min\left(m, n\right)$, and $\alpha$ is a scale factor. Since $\mathbf{U}$ and $\mathbf{V}$ define the orthonormal bases for the output and input subspaces, constraining them to the Stiefel manifold is necessary to retain their orthonormality during optimization. 
\begin{definition}[Stiefel Manifold]
The Stiefel manifold $\St \left(k, r\right)$ of orthonormal matrices in $\mathbb{R}^{k \times r}$ ($r \le k$) is defined as 
\begin{align}
    \St \left(k, r\right) \coloneqq \left\{ \mathbf{M} \in \mathbb{R}^{k \times r} \mid  {\mathbf{M}}^{\top} \mathbf{M} = \mathbf{I}_{r}\right\}.\label{eq:stiefel_definition}
\end{align}
\end{definition}

\paragraph{Optimization on the Stiefel Manifold}
Operations on the Stiefel manifold~\cite{Absil_08Optimization,Li_etal_25StelLA} requires additional choices relating to Riemannian geometry, including the local tangent space $\mathcal{T}_{\mathbf{M}} \St \left(k, r\right)$ at a point $\mathbf{M} \in \St \left(k, r\right)$ containing matrices $\mathbf{Z}$ such that $\mathbf{Z}^{\top} \mathbf{M} + \mathbf{M}^{\top} \mathbf{Z} = \mathbf{0}$ and the canonical metric $g_{\mathbf{M}} \left(\mathbf{Z}, \mathbf{T}\right) \coloneqq \tr\left(\mathbf{Z}^{\top}\left(\mathbf{I}_{k} - \frac{1}{2}\mathbf{M}\mathbf{M}^{\top}\right)\mathbf{T}\right)$ between two matrices $\mathbf{Z}, \mathbf{T} \in \mathcal{T}_{\mathbf{M}} \St \left(k, r\right)$. Given an optimization objective $\mathcal{L} : \St \left(k, r\right) \rightarrow \mathbb{R}$, the Riemannian gradient $\grad_{\mathbf{M}} \mathcal{L} \left(\mathbf{M}\right) \in \mathcal{T}_{\mathbf{M}} \St \left(k, r\right)$ of $\mathcal{L}$ can be computed from its Euclidean gradient $\nabla_{\mathbf{M}} \mathcal{L} \left(\mathbf{M}\right)$ by
\begin{align}
    \grad_{\mathbf{M}} \mathcal{L} \left(\mathbf{M}\right) \coloneqq \nabla_{\mathbf{M}} \mathcal{L} \left(\mathbf{M}\right) - \mathbf{M}^{\top} \left(\nabla_{\mathbf{M}} \mathcal{L} \left(\mathbf{M}\right)\right) \mathbf{M}.\label{eq:riemannian_grad}
\end{align}
However, the modification made by the optimizer can cause the update step to deviate from the tangent space. Hence, for an update $\Delta$ obtained after the optimizer step, we need to project it back to $\mathcal{T}_{\mathbf{M}} \St \left(k, r\right)$ using the tangent-space projection function:
\begin{align}
    \proj_{\mathbf{M}} \left(\Delta\right) \coloneqq \Delta - \mathbf{M} \symm\left(\mathbf{M}^{\top}\Delta\right),\label{eq:stiefel_tang_proj}
\end{align}
where $\symm\left(\mathbf{A}\right) = \frac{1}{2} \left(\mathbf{A}^{\top} + \mathbf{A}\right)$. After updating along the tangent space with $\Delta_{\top} \coloneqq \proj_{\mathbf{M}} \left(\Delta\right)$, the next point can be obtained from a retraction function chosen as
\begin{align}
    \retr_{\mathbf{M}} \left(\Delta_{\top}\right) \coloneqq \uf \left( \mathbf{M} + \Delta_{\top} \right),\label{eq:stiefel_retr}
\end{align}
where $\uf \left(\cdot \right)$ returns the orthogonal matrix from a polar decomposition.
These geometric operations are visualized in Fig.~\ref{fig:stiefel_opt}.

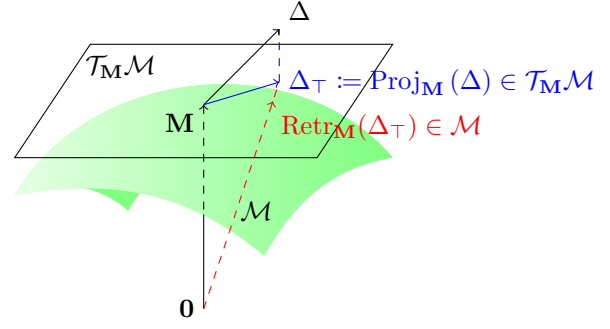
\begin{figure}[t]
    \centering
    \resizebox{!}{!}{\begin{tikzpicture}
    \draw[fill=green!50]
        (0, 0) to[out=20, in=140] (1.5, -0.2) to [out=60, in=160]
        (5, 0.5) to[out=130, in=60]
        cycle[draw=none];
    \shade[thin, left color=green!10, right color=green!50]
        (0, 0) to[out=10, in=140] (3.3, -0.8) to [out=60, in=190] (5, 0.5)
            to[out=130, in=60] cycle;
    \node at (3.2, -0.2) {$\mathcal{M}$};
    \draw[<-,dashed] (2.5, 1.2) node[anchor=north east] {$\mathbf{M}$} -- (2.5, -0.25);
    \draw (2.5, -0.25) -- (2.5, -1.5) node[anchor=east] {$\mathbf{0}$};
    \draw[->] (2.5, 1.2) -- (3.5, 2.2) node[anchor=south west] {$\Delta$};
    \draw[draw] (0, 0.5) -- (4, 0.5) -- (5, 2) -- (1, 2) -- cycle;
    \node at (1.4, 1.7) {$\mathcal{T}_{\mathbf{M}}\mathcal{M}$};
    \draw[->,blue] (2.5, 1.2) -- (3.5, 1.5) node[anchor=west] {$\Delta_{\top}\coloneqq\mathrm{Proj}_{\mathbf{M}}\left(\Delta\right) \in \mathcal{T}_{\mathbf{M}}\mathcal{M}$};
    \draw[dashed,blue] (3.5, 1.5) -- (3.5, 2.2);
    \node[anchor=north west, red] at (3.4, 1.2) {$\mathrm{Retr}_{\mathbf{M}}(\Delta_{\top}) \in \mathcal{M}$};
    \draw[dashed, red] (2.5, -1.5) -- (3.4, 1.2) -- (3.5, 1.5);
    \draw[red, ->] (3.4, 1.2) -- ($(3.4, 1.2)!0.15!(3.5, 1.5)$);
\end{tikzpicture}}
    \caption{An example of an update step on the Riemannian manifold $\mathcal{M}$. Starting from a point $\mathbf{M}$ on the manifold with a Euclidean vector $\Delta$, $\Delta$ is first projected (depicted in \textcolor{blue}{blue}) to the local tangent space $\mathcal{T}_{\mathbf{M}}\mathcal{M}$ of $\mathbf{M}$ to obtain $\Delta_{\top}$. Then, the retraction operation $\retr_{\mathbf{M}} \left(\Delta_{\top}\right)$ (depicted in \textcolor{red}{red}) is conducted to find the next point on the manifold corresponding $\Delta_{\top}$ on the tangent space.}
    \label{fig:stiefel_opt}
\end{figure}

\subsection{Stein Variational Gradient Descent for Bayesian Inference}
\paragraph{Gradient Flow in Probability Space} Let us consider a problem of efficiently sampling from a distribution $p\left(\boldsymbol{\theta}\right) \propto \exp\left(- \beta \Psi \left(\boldsymbol{\theta}\right)\right)$ over a subspace $\boldsymbol{\Theta} \subseteq \mathbb{R}^{d}$, where $\Psi \left(\cdot\right)$ is the energy function. In this case, $p$ is the solution of optimization problem:
\begin{align}
    \min_{\rho \in \mathcal{P} \left(\boldsymbol{\Theta}\right)} \left\{\mathcal{F}\left(\rho\right) \coloneqq \int \beta\Psi d\rho + \int \log{\rho} d\rho\right\},
\end{align}
where $\mathcal{P}\left(\boldsymbol{\Theta}\right)$ is the space of distributions over $\boldsymbol{\Theta}$ with the Wasserstein distance~\cite{Ambrosio_etal_08Gradient}. The gradient flow of $\mathcal{F}$ in the Wasserstein space~\cite{Ambrosio_etal_08Gradient} at time $s$ is described by the continuity equation
\begin{align}
    \partial_{s} \rho_{s} + \nabla_{\boldsymbol{\theta}} \cdot \left( \rho_{s} \nabla_{\boldsymbol{\theta}} \frac{\partial\mathcal{F}}{\partial \rho_{s}} \right) = 0,
\end{align}
where $\frac{\partial\mathcal{F}}{\partial \rho_{s}}$ is the first functional derivative of $\mathcal{F}$.

\paragraph{Steepest Descent Solution} To solve the gradient flow problem, with each a current distribution $\rho_{t}$ at the time $t$, we need to find the velocity field $T\left(\boldsymbol{\theta}\right) = \boldsymbol{\theta} + \eta \phi^{\ast}\left(\boldsymbol{\theta}\right)$ by finding the steepest descent direction:
\begin{align}
    \phi^{\ast} \coloneqq \argmin_{\phi} \left.\frac{\partial}{\partial \eta} \mathcal{F} \left(T_{\sharp} \rho_{t}\right)\right\vert_{\eta=0},
\end{align}
where $T_{\sharp} \rho_{t}$ is the push-forward distribution of $\rho_{t}$ by $T$. With an additional assumption that $\phi$ belongs to a reproducing kernel Hilbert space (RKHS) $\mathcal{H}_{\kappa}^{d}$ defined by a positive semi-definite kernel $\mathcal{\kappa} \left(\cdot, \cdot\right): \boldsymbol{\Theta} \times \boldsymbol{\Theta} \rightarrow \mathbb{R}$, Stein variational gradient descent (SVGD)~\cite{Liu_Wang_16Stein} achieves a closed-form solution for the optimal update step as follows
\begin{align}
    \phi^{\ast}\left(\boldsymbol{\theta}\right) = \int \left[- \beta \kappa\left(\boldsymbol{\theta}^{\prime}, \boldsymbol{\theta} \right) \nabla_{\boldsymbol{\theta}^{\prime}} \Psi \left(\boldsymbol{\theta}^{\prime}\right) + \nabla_{\boldsymbol{\theta}^{\prime}} \kappa \left(\boldsymbol{\theta}^{\prime}, \boldsymbol{\theta}\right) \right]d\rho_{t}\left(\boldsymbol{\theta}^{\prime}\right).
\end{align}
In practice, this formula can be approximated with $M$ particles as follows
\begin{align}
    \phi^{\ast}\left(\boldsymbol{\theta}_{j}\right) = \frac{1}{M} \sum_{i=1}^{M} \left[- \beta \kappa\left(\boldsymbol{\theta}_{i}, \boldsymbol{\theta}_{j} \right) \nabla_{\boldsymbol{\theta}_{i}} \Psi \left(\boldsymbol{\theta}_{i}\right) + \nabla_{\boldsymbol{\theta}_{i}} \kappa \left(\boldsymbol{\theta}_{i}, \boldsymbol{\theta}_{j}\right) \right].\label{eq:svgd_phi}
\end{align}

\section{StePS: Riemannian SVGD for Bayesian PEFT on the~Stiefel~Manifold}
The orthonormality offered by the Stiefel manifold allows the full utilization of the subspace rank as well as disentanglement between the dimensions. This is particularly beneficial with low-rank adapters since $r \ll \min\left(m, n\right)$. Despite the ability to enhance the expressiveness of the subspace and the effectiveness in learning, research regarding estimating the uncertainty with models having parameters on the Stiefel manifold. As mentioned in the previous section SVGD is appropriate for this task due to its flexibility. However, the canonical Euclidean SVGD cannot directly be applied to the adaptation for Bayesian inference. In this section, we introduce the necessary adjustments to the optimization problem and the transportation map, and present the necessary formulation to find the optimal solution with these adjustments.

\subsection{Gradient Flow on the Riemannian Manifold}
To define the optimization problem for $\mathbf{U}$ and $\mathbf{V}$ in the SVD-based adaptation, we consider a Riemannian manifold $\mathcal{M}$ with a corresponding set of distributions over it $\mathcal{P}\left(\mathcal{M}\right)$ in place of the Euclidean subspace $\boldsymbol{\Theta}$. The target distribution $p_{\mathcal{M}} \left(\boldsymbol{\theta}\right)$ over the manifold that we aim to sample is defined as follows
\begin{align}
    p_{\mathcal{M}} \coloneqq \argmin_{\rho \in \mathcal{P} \left(\mathcal{M}\right)} \left\{ \mathcal{F} \left(\rho\right) \coloneqq \int \beta \Psi d\rho + \int \log \rho d\rho \right\}.\label{eq:manifold_objective}
\end{align}

Similar to SVGD~\cite{Liu_Wang_16Stein}, we start from an initial prior distribution $\rho_{0}$ and iteratively transport it to match $p_{\mathcal{M}}$ via a flow of distribution $\left\{\rho_{t}\right\}_{t \ge 0}$ on the manifold. Assume that at the time step $t$, we are at the distribution $\rho_{t} \in \mathcal{P}\left(\mathcal{M}\right)$, our objective is to learn a transportation map $T$ to transport it to the next distribution $\rho_{t + 1} \coloneqq T_{\sharp} \rho_{t} \in \mathcal{P}\left(\mathcal{M}\right)$ that further reduces the objective function $\mathcal{F}\left(\cdot\right)$. 

\begin{proposition}[Transportation Map on the Riemannian Manifold]
Let the Riemannian manifold $\mathcal{M}$ be equipped with $\proj_{\boldsymbol{\theta}}$ and $\retr_{\boldsymbol{\theta}}$ respectively being the projection map to the tangent space $\mathcal{T}_{\boldsymbol{\theta}}\mathcal{M}$ at a point $\boldsymbol{\theta}$ and the retraction operation from $\mathcal{T}_{\boldsymbol{\theta}}\mathcal{M}$ back to $\mathcal{M}$. We can formulate the transportation map as follows
\begin{align}
    T \left(\boldsymbol{\theta}\right) \coloneqq \retr_{\boldsymbol{\theta}} \left( \eta \proj_{\boldsymbol{\theta}} \left( \phi \left(\boldsymbol{\theta}\right) \right) \right),\label{eq:transportation_map}
\end{align}
where $\phi$ is a function in the RKHS $\mathcal{H}_{\kappa}^{\dim\left(\mathcal{M}\right)}$ with the kernel~$\kappa$. 
\end{proposition}

\begin{algorithm*}[t]
    \caption{Riemannian \underline{\textbf{Ste}}in Variational Gradient Descent for Bayesian \underline{\textbf{P}}EFT on the \underline{\textbf{S}}tiefel Manifold (\textbf{StePS})}
    \label{alg:optimization}
    \begin{algorithmic}[1]
        \REQUIRE Training dataset $\mathcal{D}$, pre-trained weight $\mathbf{W}_{0}$, rank $r$, scale factor $\alpha$, number of particles $M$, a step function of a Euclidean optimizer $\mathrm{step}\left(\cdot,\cdot\right)$ for minimization, and number of iterations $T$
        \FOR{$i \leftarrow 1$ to $M$}
            \STATE Randomly initialize $\mathbf{U}_{i,0} \in \St\left(m, r\right)$ and $\mathbf{V}_{i,0} \in \St\left(n, r\right)$ and set $\mathbf{S}_{i,0} \leftarrow \mathbf{I}_{r}$
        \ENDFOR
        \FOR{$i \leftarrow 0$ to $T-1$}
            \FOR{$i \leftarrow 1$ to $M$}
                \STATE Compute $\mathbf{W}_{i,t} \leftarrow \mathbf{W}_{0} + \frac{\alpha}{r} \mathbf{U}_{i,t} \mathbf{S}_{i,t} \mathbf{V}_{i,t}^{\top}$
                \STATE Compute the loss $\mathcal{L}_{\mathcal{D}} \left(\mathbf{W}_{i,t}\right)$ using~\eqref{eq:loss}
                \STATE Compute the Euclidean gradients $\nabla_{\mathbf{U}_{i,t}}\mathcal{L}_{\mathcal{D}} \left(\mathbf{W}_{i,t}\right)$, $\nabla_{\mathbf{S}_{i,t}}\mathcal{L}_{\mathcal{D}} \left(\mathbf{W}_{i,t}\right)$, and $\nabla_{\mathbf{V}_{i,t}}\mathcal{L}_{\mathcal{D}} \left(\mathbf{W}_{i,t}\right)$
                \STATE Compute Riemannian gradients $\grad_{\mathbf{U}_{i,t}}\mathcal{L}_{\mathcal{D}} \left(\mathbf{W}_{i,t}\right)$ and $\grad_{\mathbf{V}_{i,t}}\mathcal{L}_{\mathcal{D}} \left(\mathbf{W}_{i,t}\right)$ using~\eqref{eq:riemannian_grad}
            \ENDFOR
            \FOR{$i \leftarrow 1$ to $M$}
                \STATE Compute the update $\phi^{\ast}$ for $\mathbf{U}_{i,t}$ and $\mathbf{V}_{i,t}$ using~\eqref{eq:riemannian_svgd_phi}, and for $\mathbf{S}_{i,t}$ using~\eqref{eq:svgd_phi} with the RBF kernel in~\eqref{eq:kernel}
            \ENDFOR
            \FOR{$i \leftarrow 1$ to $M$}
                \STATE Project $\phi^{\ast}$ to the tangent space by $\phi^{\ast}_{\top}\left(\mathbf{U}_{i,t}\right) \leftarrow \proj_{\mathbf{U}_{i,t}} \left(\phi^{\ast}\left(\mathbf{U}_{i,t}\right)\right)$ and $\phi^{\ast}_{\top}\left(\mathbf{V}_{i,t}\right) \leftarrow \proj_{\mathbf{V}_{i,t}} \left(\phi^{\ast}\left(\mathbf{V}_{i,t}\right)\right)$ using~\eqref{eq:stiefel_tang_proj} 
                \STATE Compute the next values by $\tilde{\mathbf{U}}_{i,t} \leftarrow \mathrm{step} \left(\mathbf{U}_{i,t}, -\phi^{\ast}_{\top}\left(\mathbf{U}_{i,t}\right)\right)$, ${\mathbf{S}}_{i,t+1} \leftarrow \mathrm{step} \left(\mathbf{S}_{i,t}, -\phi^{\ast}\left(\mathbf{S}_{i,t}\right)\right)$, and $\tilde{\mathbf{V}}_{i,t} \leftarrow \mathrm{step} \left(\mathbf{V}_{i,t}, -\phi^{\ast}_{\top}\left(\mathbf{V}_{i,t}\right)\right)$
                \STATE Project the update step to the tangent spaces by $\Delta_{\top, \mathbf{U}_{i,t}} \leftarrow \proj_{\mathbf{U}_{i,t}} \left(\tilde{\mathbf{U}}_{i,t} - \mathbf{U}_{i,t}\right)$ and $\Delta_{\top, \mathbf{V}_{i,t}} \leftarrow \proj_{\mathbf{V}_{i,t}} \left(\tilde{\mathbf{V}}_{i,t} - \mathbf{V}_{i,t}\right)$ using~\eqref{eq:stiefel_tang_proj}
                \STATE Update and retract back to the manifold by $\mathbf{U}_{i,t+1}\leftarrow \retr_{\mathbf{U}_{i,t}} \left(\Delta_{\top, \mathbf{U}_{i,t}}\right)$ and $\mathbf{V}_{i,t+1}\leftarrow \retr_{\mathbf{V}_{i,t}} \left(\Delta_{\top, \mathbf{V}_{i,t}}\right)$ using~\eqref{eq:stiefel_retr}
            \ENDFOR
        \ENDFOR
        \RETURN Final particles $\left\{\mathbf{U}_{i,T}, \mathbf{S}_{i,T}, \mathbf{V}_{i,T}\right\}_{i=1}^{M}$
    \end{algorithmic}
\end{algorithm*}

\subsection{Steepest Descent Solution for the Stiefel Manifold}
\paragraph{Theoretical Development for the Riemannian Manifold} 
With the aforementioned transportation map, we aim to find the optimal $\phi^{\ast}$ by solving 
\begin{align}
    \phi^{\ast} \coloneqq \argmin_{\phi \in \mathcal{H}_{\kappa}^{\dim \left(\mathcal{M}\right)}} \left.\frac{\partial}{\partial \eta} \mathcal{F} \left(T_{\sharp} \rho_{t}\right)\right\vert_{\eta=0},
\end{align}
where $T_{\sharp} \rho_{t}$ is the distribution obtained by transporting $\rho_{t}$ through the map $T$.

Let us denote $\rho_t^{\left[T\right]} \coloneqq T_{\sharp} \rho_t$, we derive $\mathcal{F} \left(\rho_{t}^{\left[T\right]}\right)$ as follows:
\begin{align}
\mathcal{F}\left(\rho_{t}^{\left[T\right]}\right) & =\beta\int\Psi d\rho_{t}^{\left[T\right]}+\int\log\rho_{t}^{\left[T\right]}d\rho_{t}^{\left[T\right]}\nonumber \\
 & =\beta\int\Psi\left(T\left(\boldsymbol{\theta}\right)\right)d\rho_{t}+\int\log\rho_{t}^{\left[T\right]}\left(T\left(\boldsymbol{\theta}\right)\right)d\rho_{t}.\label{eq:new_obj}
\end{align}

Let us consider a sufficiently small step size $\eta>0$ so that $T$ is a bijection. As a result, using the formula of the density change for a bijection: $\rho_{t} \left(\boldsymbol{\theta}\right) = \rho_{t}^{\left[T\right]}\left(T\left(\boldsymbol{\theta}\right)\right) \left\vert \det\left(\nabla_{\boldsymbol{\theta}} T\left(\boldsymbol{\theta}\right)\right)\right\vert$,
we can rewrite the objective function of interest as:
\begin{align}
    \mathcal{F}\left(\rho_{t}^{\left[T\right]}\right) &= \beta\int\Psi\left(T\left(\boldsymbol{\theta}\right)\right)d\rho_{t} \nonumber\\
    &\phantom{{}={}} +\int\left[\log\rho_{t}\left(\boldsymbol{\theta}\right)-\log\left\vert \det \left(\nabla_{\boldsymbol{\theta}} T\left(\boldsymbol{\theta}\right)\right)\right\vert\right] d\rho_{t} \nonumber\\
    &= \,\beta\int\Psi\left(\retr_{\boldsymbol{\theta}}\left(\eta\proj_{\boldsymbol{\theta}}\left(\phi\left(\boldsymbol{\theta}\right)\right)\right)\right)d\rho_{t}+\mathrm{const} \nonumber\\
    &\phantom{{}={}} -\int\log\left\vert \det\left(\nabla_{\boldsymbol{\theta}}\mathcal{\retr}_{\boldsymbol{\theta}}\left(\eta\proj_{\boldsymbol{\theta}}\left(\phi\left(\boldsymbol{\theta}\right)\right)\right)\right)\right\vert d\rho_{t}.
    \label{eq:new_obj1}
\end{align}


Let $\mathcal{F}_{1} \left(\cdot\right)$ and $\mathcal{F}_{2} \left(\cdot\right)$ respectively denote the first and second term of $\mathcal{F}\left(\cdot\right)$ defined in \eqref{eq:new_obj1}. The derivative of the first term can be calculated as 
\begin{align}
    &\left.\frac{\partial}{\partial\eta}\mathcal{F}_{1}\left(\rho_{t}^{\left[T\right]}\right)\right\vert_{\eta=0}\nonumber \\
    &= \beta \int \left< \nabla_{\boldsymbol{\theta}} \Psi\left(\boldsymbol{\theta}\right), D_{\eta} \retr_{\boldsymbol{\theta}}\left(\mathbf{0}\right) \left[\proj_{\boldsymbol{\theta}} \left(\phi\left(\boldsymbol{\theta}\right)\right) \right]\right> d\rho_{t}\nonumber\\
    &= \beta\int\left< \nabla_{\boldsymbol{\theta}}\Psi\left(\boldsymbol{\theta}\right),\proj_{\boldsymbol{\theta}}\left(\phi\left(\boldsymbol{\theta}\right)\right)\right> d\rho_{t}\nonumber\nonumber \\
    &=  \beta\int\left< \proj_{\boldsymbol{\theta}}\left(\nabla_{\boldsymbol{\theta}}\Psi\left(\boldsymbol{\theta}\right)\right),\phi\left(\boldsymbol{\theta}\right)\right> d\rho_{t},\label{eq:F1_div}
\end{align}
where $\proj_{\boldsymbol{\theta}} \left(\nabla_{\boldsymbol{\theta}}\Psi\left(\boldsymbol{\theta}\right)\right)$ can be computed using the Riemannian gradient in \eqref{eq:riemannian_grad}.
Furthermore, the derivative of the second term can be formulated as
\begin{align*}
    \left.\frac{\partial}{\partial\eta}\mathcal{F}_{2}\left(\rho_{t}^{\left[T\right]}\right) \right\vert_{\eta=0} &= \left.\tr\left(J\left(\eta\right)^{-1} J^{\prime}\left(\eta\right)\right)\right\vert_{\eta=0} =\tr\left(\left.J^{\prime}\left(\eta\right)\right\vert_{\eta=0}\right),
\end{align*}
where $J\left(\eta\right) \coloneqq \nabla_{\boldsymbol{\theta}} \retr_{\boldsymbol{\theta}}\left(\eta\proj_{\boldsymbol{\theta}}\left(\phi\left(\boldsymbol{\theta}\right)\right)\right)$ with $J\left(\mathbf{0}\right) = \mathbf{I}$. It appears that
\begin{align}
    \left.J^{\prime}\left(\eta\right)\right\vert_{\eta=0} &= \left.\frac{\partial}{\partial\eta} \nabla_{\boldsymbol{\theta}} \retr_{\boldsymbol{\theta}}\left(\eta\proj_{\boldsymbol{\theta}}\left(\phi\left(\boldsymbol{\theta}\right)\right)\right) \right\vert_{\eta=0}\nonumber\\
     & =\nabla_{\boldsymbol{\theta}} \left[\left.\frac{\partial}{\partial\eta}\retr_{\boldsymbol{\theta}}\left(\eta\proj_{\boldsymbol{\theta}}\left(\phi\left(\boldsymbol{\theta}\right)\right)\right)\right\vert_{\eta=0}\right]\nonumber\\
     & =\nabla_{\boldsymbol{\theta}} \left[D_{\eta}\retr_{\boldsymbol{\theta}}\left(\mathbf{0}\right)\left[\proj_{\boldsymbol{\theta}}\left(\phi\left(\boldsymbol{\theta}\right)\right)\right]\right]\nonumber\\
     & =\nabla_{\boldsymbol{\theta}} \mathcal{\proj}_{\boldsymbol{\theta}}\left(\phi\left(\boldsymbol{\theta}\right)\right).\label{eq:deri_{e}ta}
\end{align}
Therefore, we can reach
\begin{align}
    \left.\frac{\partial}{\partial\eta}\mathcal{F}\left(\rho_{t}^{\left[T\right]}\right) \right\vert_{\eta=0} &= \beta\int\left< \proj_{\boldsymbol{\theta}}\left(\nabla_{\boldsymbol{\theta}}\Psi\left(\boldsymbol{\theta}\right)\right),\phi\left(\boldsymbol{\theta}\right)\right> d\rho_{t} \nonumber\\
     &\phantom{={}} -\int\tr\left(\nabla_{\boldsymbol{\theta}}\proj_{\boldsymbol{\theta}}\left(\phi\left(\boldsymbol{\theta}\right)\right)\right)d\rho_{t}.\label{eq:final_derivative}
\end{align}

\paragraph{Theoretical Development for the Stiefel Manifold}
Without loss of generality, we assume that $\boldsymbol{\theta} \in \text{St}(k,r)$ is a matrix lying on the Stiefel manifold. Additionally, more complex model parameter can be thought as a collection of matrices on the Stiefel manifolds.  

Through some manipulations (see Appendix~\ref{sec:proof_solution}), we arrive at the following result:
\begin{align}
    \left.\frac{\partial}{\partial\eta} \mathcal{F} \left(\rho_{t}^{\left[T\right]}\right)\right\vert_{\eta=0} &= \left< u_{t}\left(\cdot\right), \phi\left(\cdot\right) \right>_{\mathcal{H}_{\kappa}^{\dim\left(\mathcal{M}\right)}}, \nonumber
\end{align}
where we have found that
\begin{align}
    u_{t}\left(\cdot\right)&= \int \left[\beta \kappa \left(\boldsymbol{\theta}, \cdot\right)
    \proj_{\boldsymbol{\theta}} \left( \nabla_{\boldsymbol{\theta}}\Psi\left(\boldsymbol{\theta}\right)\right) - \proj_{\boldsymbol{\theta}} \left(\nabla_{\boldsymbol{\theta}} \kappa \left(\boldsymbol{\theta}, \cdot \right)\right)\right] d\rho_{t}. \nonumber
\end{align}
As a result, the steepest descent is satisfied when $\phi^{\ast} \left(\cdot\right) = - u_{t}\left(\cdot\right)$. This result is formally presented in Theorem~\ref{thm:solution} below.

\begin{theorem}[Steepest Descent for the Stiefel Manifold] \label{thm:solution}
We can choose the optimal solution $\phi^{\ast}$ for the transportation map in \eqref{eq:transportation_map} to solve the optimization problem in~\eqref{eq:manifold_objective} as follows
\begin{align}
    \phi^{\ast} \left(\boldsymbol{\theta}\right) &= \int \left[- \beta \kappa \left(\boldsymbol{\theta}^{\prime}, \boldsymbol{\theta}\right)  \proj_{\boldsymbol{\theta}^{\prime}} \left( \nabla_{\boldsymbol{\theta}^{\prime}}\Psi\left(\boldsymbol{\theta}^{\prime}\right)\right)\right.\nonumber\\
    &\phantom{{}={}\int[} + \left.\proj_{\boldsymbol{\theta}^{\prime}} \left(\nabla_{\boldsymbol{\theta}^{\prime}} \kappa \left(\boldsymbol{\theta}^{\prime}, \boldsymbol{\theta} \right)\right)\right] d\rho_{t} \left(\boldsymbol{\theta}^{\prime}\right).\label{eq:optimal_solution}
\end{align}
\end{theorem}

\subsection{Practical Implementation}
To practically implement the expectation in the optimal solution in \eqref{eq:optimal_solution}, we follow SVGD~\cite{Liu_Wang_16Stein} and  approximate with $M$ particles $\left\{\boldsymbol{\theta}_{i}\right\}_{i=1}^{M}$ as 
\begin{align}
    \hat{\phi}^{\ast} \left(\boldsymbol{\theta}_{j}\right) &= \frac{1}{M} \sum_{i=1}^{M}\left[- \beta \kappa \left(\boldsymbol{\theta}_{i}, \boldsymbol{\theta}_{j}\right) \proj_{\boldsymbol{\theta}_{i}} \left( \nabla_{\boldsymbol{\theta}_{i}}\Psi\left(\boldsymbol{\theta}_{i}\right)\right)\right. \nonumber\\
    &\phantom{{}={} \frac{1}{M} \sum_{i=1}^{N} {}[{}}+ \left.\proj_{\boldsymbol{\theta}_{i}} \left(\nabla_{\boldsymbol{\theta}_{i}} \kappa \left(\boldsymbol{\theta}_{i}, \boldsymbol{\theta}_{j} \right)\right)\right].\label{eq:riemannian_svgd_phi}
\end{align}

Regarding the choice for the kernel $\kappa$, we opt for the radial basis function (RBF) kernel:
\begin{align}
    \kappa_{\mathrm{RBF}} \left(\boldsymbol{\theta}^{\prime}, \boldsymbol{\theta}\right) \coloneqq \exp\left( -\frac{\left\Vert \boldsymbol{\theta}^{\prime} - \boldsymbol{\theta} \right\Vert^{2}_{F}}{h} \right),\label{eq:kernel}
\end{align}
where $h > 0$ is the bandwidth of the kernel and $\left\Vert \cdot \right\Vert_{F}$ denotes the Frobenius norm. Note that this choice is equivalent up to a multiplicative constant to the commonly used von Mises--Fisher (vMF) kernel $\kappa_{\mathrm{vMF}}  \left(\boldsymbol{\theta}^{\prime}, \boldsymbol{\theta}\right) \coloneqq\exp\left(\gamma \tr\left({\boldsymbol{\theta}^{\prime}}^{\top} \boldsymbol{\theta}\right)\right)$ with $\gamma = 2 / h$ for the Stiefel manifold~\cite{Liu_Zhu_18Riemannian}. Consequently, this choice enables a unified kernel formulation that applies consistently to parameters on both the Stiefel manifold and the Euclidean space. In particular, we compute a joint kernel for all parameters (i.e., $\mathbf{U}$, $\mathbf{S}$, and $\mathbf{V}$) of the adapters and perform additional Stiefel tangent-space projection and retraction for $\mathbf{U}$ and $\mathbf{V}$. Regarding the value for the kernel bandwidth $h$, we adaptively evaluate the bandwidth at each step with $h \coloneqq \med^{2}/\log M$ as proposed in SVGD~\cite{Liu_Wang_16Stein}, where $\med$ is the median of the pairwise distances between each pair $\boldsymbol{\theta}_{i}$ and $\boldsymbol{\theta}_{j}$ ($i \neq j$). 

The selection for the energy function $\Psi$ in \eqref{eq:manifold_objective} is the empirical loss $\mathcal{L}_{\mathcal{D}} \left(\mathbf{W}\right)$ computed over a training set $\mathcal{D} \coloneqq \left\{\left(\mathbf{x}_{i}, y_{i}\right)\right\}_{i=1}^{N}$ as follows:
\begin{align}
    \mathcal{L}_{\mathcal{D}} \left(\mathbf{W}\right) \coloneqq \frac{1}{N} \sum_{i=1}^{N} \ell \left(f_{\mathbf{W}}\left(\mathbf{x}_{i}\right), y_{i}\right),\label{eq:loss}
\end{align}
where $f_{\mathbf{W}}\left(\mathbf{x}_{i}\right)$ is the prediction made by the model $f$ with the parameters $\mathbf{W}$ and $\ell$ is a loss function depending on the current task (e.g., cross-entropy loss in classification task). The detailed optimization process of our method is presented in Algorithm~\ref{alg:optimization}.

\section{Experiments}
\subsection{Experimental Settings}
\paragraph{Fine-Tuning Task}
Following the experimental set-ups of prior works~\cite{Yang_etal_24Bayesian,Wang_etal_24BLoB,Samplawski_etal_25Scalable}, we evaluate our approach using a collection of commonsense reasoning benchmarks. We fine-tune \texttt{Llama2-7B}~\cite{Touvron_etal_23Llama2} to answer multiple-choice questions where the answers are evaluated via next-token logits corresponding to possible answers for each dataset. The fine-tuning objective is minimizing the Negative Log-Likelihood of the correct answer tokens. 

We consider two settings in the evaluation: in-distribution and out-of-distribution. The former one involves six tasks: Winogrande-Small (WG\nobreakdash-S) and Winogrande-Medium (WG\nobreakdash-M)~\cite{Sakaguchi_etal21Winogrande}, ARC-Challenge (ARC\nobreakdash-C) and ARC-Easy (ARC\nobreakdash-E)~\cite{Clark_etal_18Think}, OpenBookQA (OBQA)~\cite{Mihaylov_etal_18Can}, and BoolQ~\cite{Clark_etal_19BoolQ}. For each of them, the models are fine-tuned on the training set and evaluated on the test set. In the latter out-of-distribution setting, the models are fine-tuned on the training set of OBQA and evaluated on the test sets of ARC\nobreakdash-C and ARC\nobreakdash-E, as well as on the MMLU-Chemistry and MMLU-Physics~\cite{Hendrycks_etal_21Measuring} to assess the generalization ability under smaller and larger distribution shifts. The details of these evaluation datasets are available in Table~\ref{tab:dataset}.

\paragraph{Evaluation Metrics}
For evaluation, we use accuracy to evaluate the ability of the models to correctly answer the questions. In addition, we utilize the Expected Calibration Error (ECE)~\cite{Guo_etal_17Calibration} and Negative Log-Likelihood (NLL) to evaluate the uncertainty estimation ability of these models. Higher accuracy, and lower ECE and NLL values are preferable, which demonstrate that a model can provide accurate predictions without overconfidence. 

\paragraph{Baseline Methods}
We compare our StePS framework against a range of uncertainty quantification methods for LoRA~\cite{Hu_etal_22LoRA} and StelLA~\cite{Li_etal_25StelLA}. For the standard LoRA and StelLA training, we report two deterministic configurations: Maximum Likelihood Estimation (MLE), without weight decay regularization, and Maximum A Posteriori (MAP) estimation, with weight decay regularization. We further include standard Bayesian deep learning baselines for both LoRA and StelLA, consisting of Monte Carlo Dropout (MCD)~\cite{Gal_Ghahramani_16Dropout} and Deep Ensemble~\cite{Balabanov_25Uncertainty, Lakshminarayanan_etal_17Simple}. For recent state-of-the-art approaches for LoRA, we consider Laplace-LoRA~\cite{Yang_etal_24Bayesian}, BLoB~\cite{Wang_etal_24BLoB}, and ScalaBL~\cite{Samplawski_etal_25Scalable} for comparison. Finally, as a particle-based variational inference baseline, we include Stein variational gradient descent (SVGD)~\cite{Liu_Wang_16Stein} applied in the Euclidean parameter space for LoRA.

\paragraph{Implementation Details}
We incorporate the adapters into the query and value projections of each self-attention layer, as well as into the softmax output head of the LLM as in prior works~\cite{Yang_etal_24Bayesian,Wang_etal_24BLoB,Samplawski_etal_25Scalable}. For both LoRA and StelLA, we choose the rank $r=16$ and the scale factor $\alpha=32$. All approaches are trained for $5,000$ optimization steps using the \texttt{AdamW} optimizer~\cite{Kingma_Ba_14Adam,Loshchilov_Hutter_19Decoupled}, with a batch size of $4$, a linear learning rate scheduler, a warm-up ratio of $0.06$, and the maximum learning rate of $10^{-4}$. For MAP baselines, a weight decay value of $10^{-5}$ is applied. During training and evaluation, the frozen base weights of the models are quantized to $8$-bit, while trainable parameters remain in $32$-bit precision. Except for MLE and MAP baselines, all methods are evaluated using $4$ inference samples or particles to ensure a leveling ground for comparison among stochastic and particle-based methods. All remaining method-specific hyperparameters are set to the default values specified by each method. All experiments are performed on a single 32GB NVIDIA V100 or 48GB NVIDIA L40S GPU with $3$ random initializations. The implementation of StePS is available at \href{https://github.com/quangdzuytran/StePS}{\texttt{github.com/quangdzuytran/StePS}}. Additional experimental results are available at Appendix~\ref{sec:additional_exp_results}.

\begin{table}[t]
    \caption{Overview of the Commonsense Reasoning Datasets Used in~the~Experiments.}
    \label{tab:dataset}
    \centering
    \begin{tabular}{lccc}
        \toprule
        \textbf{Dataset} & \textbf{No. Classes} & \textbf{Train Size} & \textbf{Test Size} \\
        \midrule
        WG-S~\cite{Sakaguchi_etal21Winogrande} & $2$ & $0.64\mathrm{K}$ & $1.27\mathrm{K}$ \\
        WG-M~\cite{Sakaguchi_etal21Winogrande} & $2$ & $2.56\mathrm{K}$ & $1.27\mathrm{K}$ \\
        ARC-C~\cite{Clark_etal_18Think} & $4$ & $1.12\mathrm{K}$ & $0.30\mathrm{K}$ \\
        ARC-E~\cite{Clark_etal_18Think} & $4$ & $2.25\mathrm{K}$ & $0.57\mathrm{K}$ \\
        OBQA~\cite{Mihaylov_etal_18Can} & $4$ & $4.96\mathrm{K}$ & $0.50\mathrm{K}$ \\
        BoolQ~\cite{Clark_etal_19BoolQ} & $2$ & $2.49\mathrm{K}$ & $3.27\mathrm{K}$ \\
        MMLU-Chemistry*~\cite{Hendrycks_etal_21Measuring} & $4$ & $-$ & $0.10\mathrm{K}$ \\
        MMLU-Physics*~\cite{Hendrycks_etal_21Measuring} & $4$ & $-$ & $0.10\mathrm{K}$ \\
        \bottomrule
        \multicolumn{4}{@{}l@{}}{*MMLU datasets are only used for out-of-distribution evaluations.}
    \end{tabular}
\end{table}

\begin{table*}[t]
    \caption{In-Distribution Experiments Using \texttt{Llama2-7B}.}
    \label{tab:exp_id}
    \centering
    \begin{tabular}{llcccccc}
        \toprule
        \multirow{2}[1]{*}{\textbf{Adapter}} & \multirow{2}[1]{*}{\textbf{Method}} & \multicolumn{6}{c}{\textbf{Datasets}}\\
        \cmidrule(lr){3-8}
        & & \textbf{WG-S} & \textbf{ARC-C} & \textbf{ARC-E} & \textbf{WG-M} & \textbf{OBQA} & \textbf{BoolQ}\\
        \midrule
        \multicolumn{8}{c}{Accuracy ($\uparrow$)}\\
        \midrule
        \multirow{8}{*}{LoRA} & MLE & $70.94\pm\phantom{0}0.65$ & $68.58\pm\phantom{0}2.63$ & $86.44\pm\phantom{0}0.87$ & $76.66\pm\phantom{0}1.04$ & $82.13\pm\phantom{0}0.41$ & $88.56\pm\phantom{0}0.28$\\
        & MAP & $70.78\pm\phantom{0}1.32$ & $68.58\pm\phantom{0}1.72$ & $86.21\pm\phantom{0}0.94$ & $76.29\pm\phantom{0}0.77$ & $82.00\pm\phantom{0}0.43$ & $88.44\pm\phantom{0}0.28$\\
        & MCD & $70.75\pm\phantom{0}0.13$ & $69.37\pm\phantom{0}1.52$ & $86.15\pm\phantom{0}0.54$ & $68.28\pm12.65$ & $83.80\pm\phantom{0}1.30$ & $88.51\pm\phantom{0}0.14$\\
        & Ensemble & $72.78\pm\phantom{0}0.62$ & $\underline{70.27\pm\phantom{0}0.83}$ & $\mathbf{87.21\pm\phantom{0}0.17}$ & $77.56\pm\phantom{0}0.19$ & $83.73\pm\phantom{0}0.47$ & $\underline{89.09\pm\phantom{0}0.04}$\\
        & Laplace & $70.75\pm\phantom{0}0.52$ & $69.14\pm\phantom{0}1.52$ & $86.09\pm\phantom{0}1.63$ & $75.87\pm\phantom{0}0.79$ & $82.07\pm\phantom{0}0.93$ & $88.04\pm\phantom{0}0.13$\\
        & BLoB & $52.80\pm\phantom{0}1.99$ & $64.08\pm\phantom{0}1.77$ & $70.89\pm15.88$ & $62.29\pm\phantom{0}9.08$ & $79.67\pm\phantom{0}1.64$ & $86.84\pm\phantom{0}0.27$\\
        & ScalaBL & $71.55\pm\phantom{0}0.78$ & $68.36\pm\phantom{0}1.15$ & $87.03\pm\phantom{0}0.33$ & $77.56\pm\phantom{0}0.13$ & $83.33\pm\phantom{0}0.66$ & $88.30\pm\phantom{0}0.05$\\
        & SVGD & $73.07\pm\phantom{0}0.91$ & $69.37\pm\phantom{0}0.97$ & $86.68\pm\phantom{0}0.54$ & $78.19\pm\phantom{0}0.48$ & $\underline{84.33\pm\phantom{0}0.34}$ & $88.84\pm\phantom{0}0.17$\\
        \midrule
        \multirow{5}{*}{StelLA} & MLE & $71.36\pm\phantom{0}0.76$ & $68.47\pm\phantom{0}1.11$ & $86.44\pm\phantom{0}0.29$ & $76.77\pm\phantom{0}0.36$ & $82.67\pm\phantom{0}1.06$ & $88.60\pm\phantom{0}0.21$\\
        & MAP & $71.36\pm\phantom{0}0.76$ & $68.47\pm\phantom{0}1.11$ & $86.44\pm\phantom{0}0.29$ & $76.77\pm\phantom{0}0.36$ & $82.67\pm\phantom{0}1.06$ & $88.60\pm\phantom{0}0.21$\\
        & MCD & $71.33\pm\phantom{0}0.65$ & $68.81\pm\phantom{0}1.88$ & $86.27\pm\phantom{0}0.80$ & $77.03\pm\phantom{0}0.97$ & $81.53\pm\phantom{0}0.81$ & $88.57\pm\phantom{0}0.31$\\
        & Ensemble & $\underline{73.13\pm\phantom{0}0.87}$ & $69.82\pm\phantom{0}1.88$ & $\underline{87.09\pm\phantom{0}0.41}$ & $\underline{78.93\pm\phantom{0}0.16}$ & $83.80\pm\phantom{0}0.33$ & $\mathbf{89.18\pm\phantom{0}0.12}$\\
        & StePS (Ours) & $\mathbf{73.58\pm\phantom{0}1.13}$ & $\mathbf{70.50\pm\phantom{0}1.30}$ & $86.80\pm\phantom{0}0.66$ & $\mathbf{79.17\pm\phantom{0}0.19}$ & $\mathbf{84.50\pm\phantom{0}0.65}$ & $\mathbf{89.18\pm\phantom{0}0.08}$\\
        \midrule
        \multicolumn{8}{c}{Expected Calibration Error ($\downarrow$)}\\
        \midrule
        \multirow{8}{*}{LoRA} & MLE & $27.87\pm\phantom{0}0.55$ & $28.60\pm\phantom{0}2.81$ & $11.86\pm\phantom{0}1.27$ & $20.44\pm\phantom{0}1.20$ & $14.84\pm\phantom{0}0.83$ & $\phantom{0}4.98\pm\phantom{0}0.38$\\
        & MAP & $28.44\pm\phantom{0}1.55$ & $29.46\pm\phantom{0}1.28$ & $12.59\pm\phantom{0}0.76$ & $19.58\pm\phantom{0}0.82$ & $14.36\pm\phantom{0}0.40$ & $\phantom{0}5.04\pm\phantom{0}0.43$\\
        & MCD & $27.34\pm\phantom{0}0.60$ & $27.73\pm\phantom{0}1.45$ & $12.41\pm\phantom{0}0.59$ & $12.82\pm\phantom{0}8.88$ & $12.93\pm\phantom{0}1.11$ & $\phantom{0}4.77\pm\phantom{0}0.46$\\
        & Ensemble & $24.11\pm\phantom{0}1.11$ & $26.38\pm\phantom{0}0.82$ & $11.05\pm\phantom{0}0.27$ & $17.00\pm\phantom{0}0.57$ & $12.79\pm\phantom{0}0.32$ & $\phantom{0}4.88\pm\phantom{0}0.82$\\
        & Laplace & $28.07\pm\phantom{0}0.35$ & $28.28\pm\phantom{0}1.71$ & $12.62\pm\phantom{0}1.51$ & $21.01\pm\phantom{0}0.90$ & $14.06\pm\phantom{0}0.81$ & $\phantom{0}5.20\pm\phantom{0}0.14$\\
        & BLoB & $\mathbf{\phantom{0}2.72\pm\phantom{0}0.80}$ & $\mathbf{\phantom{0}7.67\pm\phantom{0}1.66}$ & $10.51\pm\phantom{0}4.29$ & $\mathbf{\phantom{0}3.37\pm\phantom{0}0.18}$ & $\mathbf{\phantom{0}6.92\pm\phantom{0}3.10}$ & $\mathbf{\phantom{0}1.51\pm\phantom{0}0.29}$\\
        & ScalaBL & $27.55\pm\phantom{0}1.00$ & $29.72\pm\phantom{0}1.03$ & $11.74\pm\phantom{0}0.62$ & $19.34\pm\phantom{0}0.48$ & $13.02\pm\phantom{0}1.09$ & $\phantom{0}5.15\pm\phantom{0}0.28$\\
        & SVGD & $19.79\pm\phantom{0}0.83$ & $21.10\pm\phantom{0}1.06$ & $\underline{10.13\pm\phantom{0}0.39}$ & $11.44\pm\phantom{0}5.29$ & $10.87\pm\phantom{0}0.37$ & $\phantom{0}4.87\pm\phantom{0}0.09$\\
        \midrule
        \multirow{5}{*}{StelLA} & MLE & $27.31\pm\phantom{0}1.29$ & $28.26\pm\phantom{0}0.87$ & $11.76\pm\phantom{0}0.35$ & $19.82\pm\phantom{0}0.61$ & $13.45\pm\phantom{0}1.08$ & $\phantom{0}5.05\pm\phantom{0}0.40$\\
        & MAP & $27.31\pm\phantom{0}1.29$ & $28.26\pm\phantom{0}0.87$ & $11.76\pm\phantom{0}0.35$ & $19.82\pm\phantom{0}0.61$ & $13.45\pm\phantom{0}1.08$ & $\phantom{0}5.05\pm\phantom{0}0.40$\\
        & MCD & $26.85\pm\phantom{0}1.05$ & $27.88\pm\phantom{0}2.74$ & $11.74\pm\phantom{0}0.80$ & $18.90\pm\phantom{0}0.90$ & $14.61\pm\phantom{0}0.76$ & $\phantom{0}4.98\pm\phantom{0}0.23$\\
        & Ensemble & $24.11\pm\phantom{0}1.17$ & $26.54\pm\phantom{0}1.05$ & $11.02\pm\phantom{0}0.12$ & $15.09\pm\phantom{0}1.15$ & $12.64\pm\phantom{0}0.69$ & $\phantom{0}4.59\pm\phantom{0}0.00$\\
        & StePS (Ours) & $\underline{17.68\pm\phantom{0}1.22}$ & $\underline{17.61\pm\phantom{0}2.70}$ & $\mathbf{\phantom{0}7.89\pm\phantom{0}0.37}$ & $\underline{11.39\pm\phantom{0}2.55}$ & $\phantom{0}\underline{9.23\pm\phantom{0}0.47}$ & $\phantom{0}\underline{4.15\pm\phantom{0}0.25}$\\
        \midrule
        \multicolumn{8}{c}{Negative Log-Likelihood ($\downarrow$)}\\
        \midrule
        \multirow{8}{*}{LoRA} & MLE & $\phantom{0}2.49\pm\phantom{0}0.34$ & $\phantom{0}2.58\pm\phantom{0}0.14$ & $\phantom{0}0.93\pm\phantom{0}0.11$ & $\phantom{0}1.11\pm\phantom{0}0.11$ & $\phantom{0}0.97\pm\phantom{0}0.06$ & $\phantom{0}0.32\pm\phantom{0}0.00$\\
        & MAP & $\phantom{0}3.09\pm\phantom{0}0.11$ & $\phantom{0}2.93\pm\phantom{0}0.08$ & $\phantom{0}1.01\pm\phantom{0}0.11$ & $\phantom{0}0.99\pm\phantom{0}0.08$ & $\phantom{0}0.93\pm\phantom{0}0.04$ & $\phantom{0}0.32\pm\phantom{0}0.01$\\
        & MCD & $\phantom{0}2.31\pm\phantom{0}0.21$ & $\phantom{0}2.68\pm\phantom{0}0.28$ & $\phantom{0}1.05\pm\phantom{0}0.03$ & $\phantom{0}0.91\pm\phantom{0}0.17$ & $\phantom{0}0.88\pm\phantom{0}0.05$ & $\phantom{0}0.32\pm\phantom{0}0.00$\\
        & Ensemble & $\phantom{0}1.80\pm\phantom{0}0.28$ & $\phantom{0}2.17\pm\phantom{0}0.10$ & $\phantom{0}0.82\pm\phantom{0}0.01$ & $\phantom{0}0.84\pm\phantom{0}0.03$ & $\phantom{0}0.81\pm\phantom{0}0.02$ & $\mathbf{\phantom{0}0.30\pm\phantom{0}0.01}$\\
        & Laplace & $\phantom{0}3.05\pm\phantom{0}0.51$ & $\phantom{0}2.50\pm\phantom{0}0.27$ & $\phantom{0}1.13\pm\phantom{0}0.05$ & $\phantom{0}1.15\pm\phantom{0}0.04$ & $\phantom{0}0.88\pm\phantom{0}0.03$ & $\phantom{0}0.32\pm\phantom{0}0.01$\\
        & BLoB & $\mathbf{\phantom{0}0.69\pm\phantom{0}0.01}$ & $\mathbf{\phantom{0}0.95\pm\phantom{0}0.03}$ & $\phantom{0}\underline{0.76\pm\phantom{0}0.34}$ & $\mathbf{\phantom{0}0.63\pm\phantom{0}0.05}$ & $\mathbf{\phantom{0}0.56\pm\phantom{0}0.04}$ & $\phantom{0}\underline{0.31\pm\phantom{0}0.01}$\\
        & ScalaBL & $\phantom{0}2.72\pm\phantom{0}0.08$ & $\phantom{0}3.05\pm\phantom{0}0.19$ & $\phantom{0}0.96\pm\phantom{0}0.03$ & $\phantom{0}1.08\pm\phantom{0}0.12$ & $\phantom{0}0.88\pm\phantom{0}0.07$ & $\phantom{0}0.32\pm\phantom{0}0.00$\\
        & SVGD & $\phantom{0}1.80\pm\phantom{0}0.25$ & $\phantom{0}1.91\pm\phantom{0}0.13$ & $\phantom{0}0.86\pm\phantom{0}0.04$ & $\phantom{0}0.78\pm\phantom{0}0.21$ & $\phantom{0}0.80\pm\phantom{0}0.04$ & $\phantom{0}0.32\pm\phantom{0}0.00$\\
        \midrule
        \multirow{5}{*}{StelLA} & MLE & $\phantom{0}2.64\pm\phantom{0}0.18$ & $\phantom{0}2.80\pm\phantom{0}0.35$ & $\phantom{0}0.99\pm\phantom{0}0.03$ & $\phantom{0}1.01\pm\phantom{0}0.06$ & $\phantom{0}0.87\pm\phantom{0}0.04$ & $\phantom{0}0.32\pm\phantom{0}0.01$\\
        & MAP & $\phantom{0}2.64\pm\phantom{0}0.18$ & $\phantom{0}2.80\pm\phantom{0}0.35$ & $\phantom{0}0.99\pm\phantom{0}0.03$ & $\phantom{0}1.01\pm\phantom{0}0.06$ & $\phantom{0}0.87\pm\phantom{0}0.04$ & $\phantom{0}0.32\pm\phantom{0}0.01$\\
        & MCD & $\phantom{0}2.47\pm\phantom{0}0.42$ & $\phantom{0}2.50\pm\phantom{0}0.20$ & $\phantom{0}0.84\pm\phantom{0}0.04$ & $\phantom{0}1.07\pm\phantom{0}0.10$ & $\phantom{0}0.94\pm\phantom{0}0.02$ & $\phantom{0}0.32\pm\phantom{0}0.00$\\
        & Ensemble & $\phantom{0}1.89\pm\phantom{0}0.09$ & $\phantom{0}2.15\pm\phantom{0}0.20$ & $\phantom{0}0.81\pm\phantom{0}0.03$ & $\phantom{0}0.75\pm\phantom{0}0.06$ & $\phantom{0}0.79\pm\phantom{0}0.03$ & $\mathbf{\phantom{0}0.30\pm\phantom{0}0.00}$\\
        & StePS (Ours) & $\phantom{0}\underline{1.47\pm\phantom{0}0.06}$ & $\phantom{0}\underline{1.66\pm\phantom{0}0.04}$ & $\mathbf{\phantom{0}0.63\pm\phantom{0}0.01}$ & $\phantom{0}\underline{0.67\pm\phantom{0}0.12}$ & $\phantom{0}\underline{0.70\pm\phantom{0}0.03}$ & $\mathbf{\phantom{0}0.30\pm\phantom{0}0.01}$\\
        \bottomrule
        \multicolumn{8}{@{}l@{}}{}\\
        \multicolumn{8}{@{}l@{}}{The reported results are mean\textpm{}std of the metrics evaluated on the test sets over $3$ random initializations.}\\
        \multicolumn{8}{@{}l@{}}{\textbf{Bold} and \underline{underlined} results denote the best and second-best mean values on each metric for each dataset.}
    \end{tabular}
\end{table*}

\subsection{In-Distribution Inference Results}

The in-distribution results of all methods are presented in Table~\ref{tab:exp_id}. Across six benchmarks, Deep Ensemble, SVGD, and StePS consistently achieve higher accuracy than other benchmarks. Among these approaches, StePS attains relatively low Expected Calibration Errors (ECE) and Negative Log-Likelihood (NLL), achieving the lowest values on several datasets and the second-lowest on most others. Notably, StePS maintains strong calibration while simultaneously achieving the highest accuracy on five out of six datasets. Compared to the MLE and MAP variants of StelLA, StePS also exhibits improvements in all evaluation metrics, which provides a clear evidence that our geometry-aware variational inference approach enhances both accuracy and reliability. Furthermore, StelLA-based methods consistently outperform their corresponding LoRA-based counterparts. This confirms the effectiveness of the Stiefel manifold constraints in learning expressive low-rank parameters. 

While Laplace-LoRA and ScalaBL score competitively high accuracy on several benchmarks, their calibration improvements are relatively minimal compared to other baselines for uncertainty quantification. Although BLoB generally obtains the competitive calibration results with lowest ECE and NLL on most benchmarks. These calibration gains are not accompanied by meaningful predictive accuracy. In contrast, our method, StePS, delivers simultaneous improvements in accuracy and calibration that provides a more balanced trade-off. Overall, these results further highlight the advantages of combining geometry-aware optimization with particle-based variational inference for parameter-efficient fine-tuning. 

\begin{table*}[t]
    \caption{Out-of-Distribution Experiments Using \texttt{Llama2-7B}.}
    \label{tab:exp_ood}
    \centering
    \begin{tabular}{llccccc}
        \toprule
        \multirow{2}[1]{*}{\textbf{Adapter}} & \multirow{2}[1]{*}{\textbf{Method}} & \textbf{In-Distribution} & \multicolumn{2}{c}{\textbf{Smaller Distribution Shift}} & \multicolumn{2}{c}{\textbf{Larger Distribution Shift}}\\
        \cmidrule(lr){3-3} \cmidrule(lr){4-5} \cmidrule{6-7}
        & & \textbf{OBQA} & \textbf{ARC-C} & \textbf{ARC-E} & \textbf{MMLU-Chemistry} & \textbf{MMLU-Physics}\\
        \midrule
        \multicolumn{7}{c}{Accuracy ($\uparrow$)}\\
        \midrule
        \multirow{8}{*}{LoRA} & MLE & $82.13\pm\phantom{0}0.41$ & $70.16\pm\phantom{0}0.42$ & $75.41\pm\phantom{0}1.71$ & $34.00\pm\phantom{0}1.63$ & $22.67\pm\phantom{0}1.25$\\
        & MAP & $82.00\pm\phantom{0}0.43$ & $70.16\pm\phantom{0}0.42$ & $75.65\pm\phantom{0}2.04$ & $34.00\pm\phantom{0}1.63$ & $23.33\pm\phantom{0}0.47$\\
        & MCD & $83.80\pm\phantom{0}1.30$ & $69.03\pm\phantom{0}1.84$ & $76.53\pm\phantom{0}0.22$ & $37.67\pm\phantom{0}3.30$ & $26.33\pm\phantom{0}3.09$\\
        & Ensemble & $83.73\pm\phantom{0}0.47$ & $\mathbf{71.73\pm\phantom{0}1.11}$ & $\underline{77.58\pm\phantom{0}0.54}$ & $39.67\pm\phantom{0}2.62$ & $27.33\pm\phantom{0}1.25$\\
        & Laplace & $82.07\pm\phantom{0}0.93$ & $70.16\pm\phantom{0}2.25$ & $76.23\pm\phantom{0}0.38$ & $36.33\pm\phantom{0}1.89$ & $28.67\pm\phantom{0}5.19$\\
        & BLoB & $79.67\pm\phantom{0}1.64$ & $68.98\pm\phantom{0}1.84$ & $75.89\pm\phantom{0}1.16$ & $39.24\pm\phantom{0}2.73$ & $31.25\pm\phantom{0}0.00$\\
        & ScalaBL & $83.33\pm\phantom{0}0.66$ & $69.21\pm\phantom{0}1.28$ & $76.67\pm\phantom{0}1.19$ & $36.46\pm\phantom{0}2.55$ & $29.51\pm\phantom{0}2.73$\\
        & SVGD & $\underline{84.33\pm\phantom{0}0.34}$ & $68.81\pm\phantom{0}1.77$ & $76.53\pm\phantom{0}0.33$ & $33.67\pm\phantom{0}3.86$ & $31.00\pm\phantom{0}2.16$\\
        \midrule
        \multirow{5}{*}{StelLA} & MLE & $82.67\pm\phantom{0}1.06$ & $70.05\pm\phantom{0}1.39$ & $77.00\pm\phantom{0}0.84$ & $39.00\pm\phantom{0}2.16$ & $27.33\pm\phantom{0}3.40$\\
        & MAP & $82.67\pm\phantom{0}1.06$ & $\underline{71.28\pm\phantom{0}1.10}$ & $77.00\pm\phantom{0}0.33$ & $\underline{40.00\pm\phantom{0}3.56}$ & $29.33\pm\phantom{0}0.94$\\
        & MCD & $81.53\pm\phantom{0}0.81$ & $69.03\pm\phantom{0}0.57$ & $77.29\pm\phantom{0}1.98$ & $33.33\pm\phantom{0}3.40$ & $31.00\pm\phantom{0}3.56$\\
        & Ensemble & $83.80\pm\phantom{0}0.33$ & $70.95\pm\phantom{0}0.28$ & $77.40\pm\phantom{0}0.50$ & $39.33\pm\phantom{0}0.47$ & $\underline{31.33\pm\phantom{0}2.87}$\\
        & StePS (Ours) & $\mathbf{84.50\pm\phantom{0}0.65}$ & $70.50\pm\phantom{0}0.69$ & $\mathbf{77.76\pm\phantom{0}0.22}$ & $\mathbf{40.67\pm\phantom{0}1.70}$ & $\mathbf{32.67\pm\phantom{0}4.11}$\\
        \midrule
        \multicolumn{7}{c}{Expected Calibration Error ($\downarrow$)}\\
        \midrule
        \multirow{8}{*}{LoRA} & MLE & $14.84\pm\phantom{0}0.83$ & $23.55\pm\phantom{0}0.77$ & $18.08\pm\phantom{0}1.34$ & $29.69\pm\phantom{0}1.11$ & $39.29\pm\phantom{0}0.38$\\
        & MAP & $14.36\pm\phantom{0}0.40$ & $23.64\pm\phantom{0}0.73$ & $17.90\pm\phantom{0}1.60$ & $30.19\pm\phantom{0}0.41$ & $38.75\pm\phantom{0}1.14$\\
        & MCD & $12.93\pm\phantom{0}1.11$ & $23.72\pm\phantom{0}2.15$ & $16.74\pm\phantom{0}0.16$ & $27.45\pm\phantom{0}3.33$ & $33.54\pm\phantom{0}3.51$\\
        & Ensemble & $12.79\pm\phantom{0}0.32$ & $21.54\pm\phantom{0}1.39$ & $15.44\pm\phantom{0}0.65$ & $25.36\pm\phantom{0}2.00$ & $32.55\pm\phantom{0}0.41$\\
        & Laplace & $14.06\pm\phantom{0}0.81$ & $22.11\pm\phantom{0}1.01$ & $16.73\pm\phantom{0}0.74$ & $27.04\pm\phantom{0}0.96$ & $31.99\pm\phantom{0}6.02$\\
        & BLoB & $\mathbf{\phantom{0}6.92\pm\phantom{0}3.10}$ & $\mathbf{12.01\pm\phantom{0}1.24}$ & $\mathbf{\phantom{0}8.08\pm\phantom{0}0.87}$ & $\mathbf{13.96\pm\phantom{0}0.50}$ & $\mathbf{18.10\pm\phantom{0}1.51}$\\
        & ScalaBL & $13.02\pm\phantom{0}1.09$ & $23.33\pm\phantom{0}0.72$ & $16.27\pm\phantom{0}1.01$ & $26.05\pm\phantom{0}2.23$ & $31.95\pm\phantom{0}1.70$\\
        & SVGD & $10.87\pm\phantom{0}0.37$ & $23.17\pm\phantom{0}1.33$ & $15.47\pm\phantom{0}0.62$ & $29.55\pm\phantom{0}6.16$ & $30.18\pm\phantom{0}3.15$\\
        \midrule
        \multirow{5}{*}{StelLA} & MLE & $13.45\pm\phantom{0}1.08$ & $23.59\pm\phantom{0}2.28$ & $17.06\pm\phantom{0}0.75$ & $28.33\pm\phantom{0}0.79$ & $38.31\pm\phantom{0}3.33$\\
        & MAP & $13.45\pm\phantom{0}1.08$ & $21.37\pm\phantom{0}1.23$ & $16.54\pm\phantom{0}0.50$ & $27.72\pm\phantom{0}0.66$ & $35.53\pm\phantom{0}1.76$\\
        & MCD & $14.61\pm\phantom{0}0.76$ & $22.98\pm\phantom{0}0.65$ & $16.20\pm\phantom{0}1.22$ & $34.92\pm\phantom{0}5.02$ & $35.30\pm\phantom{0}3.01$\\
        & Ensemble & $12.64\pm\phantom{0}0.69$ & $21.47\pm\phantom{0}0.23$ & $13.86\pm\phantom{0}0.54$ & $23.60\pm\phantom{0}0.90$ & $31.13\pm\phantom{0}2.84$\\
        & StePS (Ours) & $\phantom{0}\underline{9.23\pm\phantom{0}0.47}$ & $\underline{18.12\pm\phantom{0}0.27}$ & $\underline{12.36\pm\phantom{0}0.26}$ & $\underline{22.34\pm\phantom{0}1.33}$ & $\underline{30.10\pm\phantom{0}2.21}$\\
        \midrule
        \multicolumn{7}{c}{Negative Log-Likelihood ($\downarrow$)}\\
        \midrule
        \multirow{8}{*}{LoRA} & MLE & $\phantom{0}0.97\pm\phantom{0}0.06$ & $\phantom{0}1.50\pm\phantom{0}0.03$ & $\phantom{0}1.16\pm\phantom{0}0.06$ & $\phantom{0}1.81\pm\phantom{0}0.12$ & $\phantom{0}1.95\pm\phantom{0}0.11$\\
        & MAP & $\phantom{0}0.93\pm\phantom{0}0.04$ & $\phantom{0}1.51\pm\phantom{0}0.04$ & $\phantom{0}1.17\pm\phantom{0}0.04$ & $\phantom{0}1.81\pm\phantom{0}0.12$ & $\phantom{0}1.95\pm\phantom{0}0.11$\\
        & MCD & $\phantom{0}0.88\pm\phantom{0}0.05$ & $\phantom{0}1.51\pm\phantom{0}0.09$ & $\phantom{0}1.10\pm\phantom{0}0.04$ & $\phantom{0}1.81\pm\phantom{0}0.06$ & $\phantom{0}1.88\pm\phantom{0}0.08$\\
        & Ensemble & $\phantom{0}0.81\pm\phantom{0}0.02$ & $\phantom{0}1.39\pm\phantom{0}0.02$ & $\phantom{0}1.02\pm\phantom{0}0.02$ & $\phantom{0}1.71\pm\phantom{0}0.02$ & $\phantom{0}1.81\pm\phantom{0}0.02$\\
        & Laplace & $\phantom{0}0.88\pm\phantom{0}0.03$ & $\phantom{0}1.43\pm\phantom{0}0.07$ & $\phantom{0}1.09\pm\phantom{0}0.01$ & $\phantom{0}1.70\pm\phantom{0}0.09$ & $\phantom{0}1.78\pm\phantom{0}0.08$\\
        & BLoB & $\mathbf{\phantom{0}0.56\pm\phantom{0}0.04}$ & $\mathbf{\phantom{0}0.90\pm\phantom{0}0.07}$ & $\mathbf{\phantom{0}0.70\pm\phantom{0}0.04}$ & $\mathbf{\phantom{0}1.43\pm\phantom{0}0.02}$ & $\mathbf{\phantom{0}1.50\pm\phantom{0}0.02}$\\
        & ScalaBL & $\phantom{0}0.88\pm\phantom{0}0.07$ & $\phantom{0}1.45\pm\phantom{0}0.04$ & $\phantom{0}1.08\pm\phantom{0}0.06$ & $\phantom{0}1.66\pm\phantom{0}0.05$ & $\phantom{0}1.74\pm\phantom{0}0.07$\\
        & SVGD & $\phantom{0}0.80\pm\phantom{0}0.04$ & $\phantom{0}1.42\pm\phantom{0}0.03$ & $\phantom{0}1.04\pm\phantom{0}0.02$ & $\phantom{0}1.87\pm\phantom{0}0.11$ & $\phantom{0}1.90\pm\phantom{0}0.13$\\
        \midrule
        \multirow{5}{*}{StelLA} & MLE & $\phantom{0}0.87\pm\phantom{0}0.04$ & $\phantom{0}1.57\pm\phantom{0}0.12$ & $\phantom{0}1.13\pm\phantom{0}0.03$ & $\phantom{0}1.77\pm\phantom{0}0.11$ & $\phantom{0}2.02\pm\phantom{0}0.04$ \\
        & MAP & $\phantom{0}0.87\pm\phantom{0}0.04$ & $\phantom{0}1.47\pm\phantom{0}0.07$ & $\phantom{0}1.07\pm\phantom{0}0.06$ & $\phantom{0}1.80\pm\phantom{0}0.03$ & $\phantom{0}1.99\pm\phantom{0}0.14$\\
        & MCD & $\phantom{0}0.94\pm\phantom{0}0.02$ & $\phantom{0}1.58\pm\phantom{0}0.12$ & $\phantom{0}1.13\pm\phantom{0}0.07$ & $\phantom{0}1.90\pm\phantom{0}0.10$ & $\phantom{0}2.04\pm\phantom{0}0.03$\\
        & Ensemble & $\phantom{0}0.79\pm\phantom{0}0.03$ & $\phantom{0}1.34\pm\phantom{0}0.04$ & $\phantom{0}0.93\pm\phantom{0}0.01$ & $\phantom{0}1.65\pm\phantom{0}0.02$ & $\phantom{0}1.80\pm\phantom{0}0.04$\\
        & StePS (Ours) & $\phantom{0}\underline{0.70\pm\phantom{0}0.03}$ & $\phantom{0}\underline{1.23\pm\phantom{0}0.03}$ & $\phantom{0}\underline{0.87\pm\phantom{0}0.02}$ & $\phantom{0}\underline{1.63\pm\phantom{0}0.04}$ & $\phantom{0}\underline{1.73\pm\phantom{0}0.06}$\\
        \bottomrule
        \multicolumn{7}{@{}l@{}}{}\\
        \multicolumn{7}{@{}l@{}}{The reported results are mean\textpm{}std of the metrics evaluated on the test sets over $3$ random initializations.}\\
        \multicolumn{7}{@{}l@{}}{\textbf{Bold} and \underline{underlined} results denote the best and second-best mean values on each metric for each dataset.}
    \end{tabular}
\end{table*}

\subsection{Out-of-Distribution Inference Results}

Table~\ref{tab:exp_ood} presents the out-of-distribution evaluation results under both smaller and larger distribution shifts. In general, all methods achieve satisfactory results on datasets with smaller distribution shifts, ARC\nobreakdash-C and ARC\nobreakdash-E, while demonstrating obvious degradation in accuracy and calibration on the datasets with larger distribution shifts, MMLU\nobreakdash-Chemistry and MMLU\nobreakdash-Physics. This behavior aligns with previous observations that domain shifts involving specialized knowledge can create a great challenge for fine-tuned LLMs~\cite{Wang_etal_24BLoB,Samplawski_etal_25Scalable}.
A consistent pattern also emerges where StelLA-based methods outperform their respective LoRA-based variants on most out-of-distribution benchmarks. This suggests that the Stiefel manifold formulation yields more robust low-rank representations under distribution shifts, especially when the domain gaps are wider as observed on the MMLU datasets.

Although StePS does not always have the highest predictive accuracy, its performance remains highly competitive in comparison to the baselines. Crucially, StePS achieves robust and consistent calibration results with the second-lowest ECE and NLL over all datasets. While BLoB obtains the lowest ECE and NLL scores, its accuracy still remains lower than other baselines on ARC\nobreakdash-C and ARC\nobreakdash-E datasets. As a result, StePS still demonstrates a more favorable balance between the predictive performance and uncertainty calibration on the out-of-distribution setting.

\section{Conclusion}
In this work, we have proposed a framework based on Riemannian Stein variational gradient descent, StePS, to both quantify the uncertainty and calibrate geometry-aware PEFT. StePS allows explicit and principled estimation of the epistemic uncertainty by iteratively transporting the particles on the Stiefel manifold toward the target distribution. To achieve this, we derive a closed-form update from the steepest descent of the associated gradient flow on the Stiefel manifold, providing a clear theoretical foundation for our method. The extensive in-distribution and out-of-distribution evaluations of our methods on diverse benchmarks demonstrate the effectiveness of our approach. StePS delivers promising results with better accuracy and comparable model calibration on in-distribution and out-of-distribution settings. In future work, we will explore model-based variational inference on the Riemannian manifold to further enhance the computational efficiency of geometry-ware Bayesian PEFT.

\bibliographystyle{IEEEtran}
\bibliography{IEEEabrv,references}

\appendices
\section{Proof of Theorem~\ref{thm:solution}}
\label{sec:proof_solution}
\begin{proof}
Let us recall the formula in~\eqref{eq:final_derivative} as follow
\begin{align}
    \left.\frac{\partial}{\partial\eta}\mathcal{F}\left(\rho_{t}^{\left[T\right]}\right) \right\vert_{\eta=0} &=\beta\int\left< \proj_{\boldsymbol{\theta}}\left(\nabla_{\boldsymbol{\theta}}\Psi\left(\boldsymbol{\theta}\right)\right),\phi\left(\boldsymbol{\theta}\right)\right> d\rho_{t}\nonumber\\
    &\phantom{{}={}} -\int\tr\left(\nabla_{\boldsymbol{\theta}}\proj_{\boldsymbol{\theta}}\left(\phi\left(\boldsymbol{\theta}\right)\right)\right)d\rho_{t}.\tag{\ref{eq:final_derivative}}
\end{align}
The first term in~\eqref{eq:final_derivative} can be derived as
\begin{align}
    &\beta\int\left< \proj_{\boldsymbol{\theta}}\left(\nabla_{\boldsymbol{\theta}}\Psi\left(\boldsymbol{\theta}\right)\right),\phi\left(\boldsymbol{\theta}\right)\right> d\rho_{t}\nonumber \\
    &= \beta\int\left< \proj_{\boldsymbol{\theta}}\left(\nabla_{\boldsymbol{\theta}}\Psi\left(\boldsymbol{\theta}\right)\right),\left< \kappa\left(\boldsymbol{\theta},\cdot\right),\phi\left(\cdot\right)\right> _{\mathcal{H}_{\kappa}^{d}}\right> d\rho_{t}\nonumber \\
    &= \beta\int\left< \kappa\left(\boldsymbol{\theta},\cdot\right)\proj_{\boldsymbol{\theta}}\left(\nabla_{\boldsymbol{\theta}}\Psi\left(\boldsymbol{\theta}\right)\right),\phi\left(\cdot\right)\right> _{\mathcal{H}_{\kappa}^{d}}d\rho_{t},\label{eq:term1}
\end{align}
where $\mathcal{H}_{\kappa}$ is the RKHS with the kernel $\kappa$ and $d = k \times r$. For the second term, we consider the derivative of the tangent-projection in \eqref{eq:stiefel_tang_proj} as
\begin{align}
    &\nabla_{\boldsymbol{\theta}}\proj_{\boldsymbol{\theta}}\left(\phi\left(\boldsymbol{\theta}\right)\right) \nonumber\\
    &=\nabla_{\boldsymbol{\theta}}\left[\phi\left(\boldsymbol{\theta}\right) - \frac{1}{2}\boldsymbol{\theta}\left[\boldsymbol{\theta}^{\top}\phi\left(\boldsymbol{\theta}\right)+\phi\left(\boldsymbol{\theta}\right)^{\top}\boldsymbol{\theta}\right]\right].\label{eq:projection_derivative}
\end{align}
The trace of its first part can be formulated as
\begin{equation}
\tr(\nabla_{\boldsymbol{\theta}}\phi\left(\boldsymbol{\theta}\right))=\left< \nabla_{\boldsymbol{\theta}}\kappa\left(\boldsymbol{\theta},\cdot\right),\phi\left(\cdot\right)\right> _{\mathcal{H}_{\kappa}^{d}}.\label{eq:term2_1}
\end{equation}
For the remaining part, let us denote $\boldsymbol{\alpha} = \boldsymbol{\theta}\boldsymbol{\theta}^T$ and $\delta_{p,q}\left(\boldsymbol{\theta}\right)=\sum_{k}\boldsymbol{\gamma}_{pk}\phi_{kq}\left(\boldsymbol{\theta}\right)$. We flatten the matrix $\boldsymbol{\alpha}\phi\left(\boldsymbol{\theta}\right)$ into a vector of $k \times r$ dimensions. The trace of its derivative becomes
\begin{align}
    \sum_{p,q}\nabla_{\theta_{pq}}\delta_{p,q}\left(\boldsymbol{\theta}\right) &=\sum_{p,q}\sum_{k}\boldsymbol{\gamma}_{pk}\nabla_{\theta_{pq}}\phi_{kq}\left(\boldsymbol{\theta}\right)\nonumber \\
    &= \sum_{k,q}\sum_{p}\boldsymbol{\gamma}_{kp}\nabla_{\theta_{kq}}\phi_{pq}\left(\boldsymbol{\theta}\right)\nonumber \\
    &= \left< \sum_{k,q}\sum_{p}\boldsymbol{\gamma}_{kp}\nabla_{\theta_{kq}}\kappa\left(\boldsymbol{\theta},\cdot\right),\phi_{pq} \left(\cdot\right)\right> _{\mathcal{H}_{\kappa}}\nonumber \\
    &= \left< \sum_{k,q}\sum_{p}\boldsymbol{\gamma}_{pk}\nabla_{\theta_{kq}}\kappa\left(\boldsymbol{\theta},\cdot\right),\phi_{pq} \left(\cdot\right)\right> _{\mathcal{H}_{\kappa}}\nonumber \\
    &= \left< \sum_{p,q}\sum_{k}\boldsymbol{\gamma}_{pk}\nabla_{\theta_{kq}}\kappa\left(\boldsymbol{\theta},\cdot\right),\phi_{pq} \left(\cdot\right)\right> _{\mathcal{H}_{\kappa}}\nonumber \\
    &= \sum_{p,q}\left< \sum_{k}\boldsymbol{\gamma}_{pk}\nabla_{\theta_{kq}}\kappa\left(\boldsymbol{\theta},\cdot\right),\phi_{pq} \left(\cdot\right)\right> _{\mathcal{H}_{\kappa}}\nonumber \\
    &= \left< \boldsymbol{\theta}\boldsymbol{\theta}^{\top}\nabla_{\boldsymbol{\theta}}\kappa\left(\boldsymbol{\theta},\cdot\right),\phi \left(\cdot\right)\right> _{\mathcal{H}_{\kappa}^{d}}.\label{eq:term2_2_1}
\end{align}
Let us denote $\lambda_{pq}\left(\boldsymbol{\theta}\right)=\sum_{k,t}\theta_{pk}\phi_{tk}\left(\boldsymbol{\theta}\right)\theta_{tq}$, the trace of its derivative can be formulated as
\begin{align}
    \sum_{p,q}\nabla_{\theta_{pq}\lambda_{pq}}\left(\boldsymbol{\theta}\right)
    &=\sum_{p,q}\sum_{k,t}\theta_{pk}\nabla_{\theta_{pq}}\phi_{tk}\left(\boldsymbol{\theta}\right)\theta_{tq}\nonumber \\
    &= \sum_{t,k}\sum_{p,q}\theta_{tq}\nabla_{\theta_{tk}}\phi_{pq}\left(\boldsymbol{\theta}\right)\theta_{pk}\nonumber \\
    &= \sum_{p,q}\left< \sum_{t,k}\theta_{tq}\nabla_{\theta_{tk}}\kappa\left(\boldsymbol{\theta},\cdot\right)\theta_{pk},\phi\left(\cdot\right)\right> _{\mathcal{H}_{\kappa}}\nonumber \\
    &= \left< \boldsymbol{\theta}\nabla_{\boldsymbol{\theta}}\kappa\left(\boldsymbol{\theta},\cdot\right)^{\top}\boldsymbol{\theta},\phi\left(\cdot\right)\right> _{\mathcal{H}_{\kappa}^{d}}.\label{eq:term2_2_2}
\end{align}
Combining~\eqref{eq:term2_1}--\eqref{eq:term2_2_2} yields the trace of \eqref{eq:projection_derivative} as follow
\begin{align}
    &\tr\left( \nabla_{\boldsymbol{\theta}}\proj_{\boldsymbol{\theta}}\left(\phi\left(\boldsymbol{\theta}\right)\right) \right) \nonumber \\
    &=\left< \nabla_{\boldsymbol{\theta}}\kappa\left(\boldsymbol{\theta},\cdot\right) - \frac{1}{2}\boldsymbol{\theta} 
\left[\boldsymbol{\theta}^{\top}\nabla_{\boldsymbol{\theta}}\kappa\left(\boldsymbol{\theta},\cdot\right)+\nabla_{\boldsymbol{\theta}}\kappa\left(\boldsymbol{\theta},\cdot\right)^{\top}\boldsymbol{\theta}\right], \phi\left(\cdot\right) \right>_{\mathcal{H}_{\kappa}^{d}} \nonumber\\
    &=\left<\proj_{\boldsymbol{\theta}}\left(\nabla_{\boldsymbol{\theta}}\kappa\left(\boldsymbol{\theta},\cdot\right)\right), \phi\left(\cdot\right)\right>_{\mathcal{H}_{\kappa}^{d}}.\label{eq:term2}
\end{align}
Finally, from~\eqref{eq:term1} and~\eqref{eq:term2}, we reach the following form of~\eqref{eq:final_derivative}:
\begin{align}
\left.\frac{\partial}{\partial\eta}\mathcal{F}\left(\rho_{t}^{\left[T\right]}\right)\right\vert_{\eta=0} &= \int\left< \left[\beta\kappa\left(\boldsymbol{\theta},\cdot\right)\proj_{\boldsymbol{\theta}}\left(\nabla_{\boldsymbol{\theta}}\Psi\left(\boldsymbol{\theta}\right)\right)\right.\right. \nonumber\\
&\phantom{{}={}}-\left.\left.\proj_{\boldsymbol{\theta}}\left(\nabla_{\boldsymbol{\theta}}\kappa\left(\boldsymbol{\theta},\cdot\right)\right)\right],\phi\left(\cdot\right)\right>_{\mathcal{H}_{\kappa}^{d}} d\rho_{t}.\label{eq:final}
\end{align}
This concludes our proof as the rest is obvious.
\end{proof}

\section{Additional Experimental Results}
\label{sec:additional_exp_results}

\subsection{Runtime}
Table~\ref{tab:runtime} presents the estimated runtime (both training and evaluation) for all methods on the Winogrande-Small (WG-S)~\cite{Sakaguchi_etal21Winogrande} dataset. From these results, it is apparent that model-based variational inference baselines, Laplace-LoRA, BLoB, and ScalaBL, have an advantage in their scalability due to their inherent design. Deep Ensemble, SVGD, and StePS, require substantially longer runtime due to the need of evaluating multiple solutions in every forward pass. In addition, geometry-aware optimizations on the Stiefel manifold demand more computation with the tangent-space projection and retraction. However, these extra costs are minimal compared to the costs of ensemble evaluations. 

\begin{table}[H]
    \caption{Estimated Runtime Comparison on WG-S Dataset.}
    \label{tab:runtime}
    \centering
    \begin{tabular}{llc}
        \toprule
        \textbf{Adapter} & \textbf{Method} & \textbf{Runtime (h)} \\
        \midrule
        \multirow{6}{*}{LoRA} & MLE/MAP/MCD & $\phantom{0}1.4$\\
        & Ensemble & $10.7$\\
        & Laplace & $\phantom{0}1.7$\\
        & BLoB & $\phantom{0}2.8$\\
        & ScalaBL & $\phantom{0}1.6$\\
        & SVGD & $\phantom{0}9.1$\\
        \midrule
        \multirow{3}{*}{StelLA} & MLE/MAP/MCD & $\phantom{0}1.7$\\
        & Ensemble & $11.1$\\
        & StePS (Ours) & $10.5$\\
        \bottomrule
    \end{tabular}
\end{table}

\subsection{Ablation Study on the Number of Particles}
We train StePS on the Winogrande-Small (WG-S)~\cite{Sakaguchi_etal21Winogrande} dataset for $4$ epochs with $M=1,2,4,6,\text{and }8$ particles to examine the effect of number of particles on the model performance. The accuracy, ECE, NLL, and estimated runtime results over $3$ random restarts are presented in Table~\ref{tab:abl_n_parts}. From these results, it is apparent that the accuracy of the model increases with more particles involved in the approximation. However, starting from $6$ particles, the ECE scores also begin to increase with the NLL values slightly fluctuating. Overall, $M=4$ achieves the best trade-off between accuracy and calibration. Additionally, using a lower number of particles also facilitates the training process with shorter runtime and reduced computational cost.

\begin{table}[H]
    \caption{Ablation Study on the Number of Particles on WG-S Dataset.}
    \label{tab:abl_n_parts}
    \centering
    \begin{tabular}{ccccc}
        \toprule
        $M$ & \textbf{Acc (\textuparrow)} & \textbf{ECE (\textdownarrow)} & \textbf{NLL (\textdownarrow)} & \textbf{Runtime (h)} \\
        \midrule
        $1$ & $66.38\pm0.58$ & $20.49\pm1.74$ & $0.88\pm0.08$ & $0.47$ \\
        $2$ & $68.88\pm0.78$ & $13.13\pm0.61$ & $0.70\pm0.01$ & $0.93$ \\
        $4$ & $69.22\pm0.84$ & $11.91\pm0.27$ & $0.66\pm0.01$ & $1.55$ \\
        $6$ & $70.02\pm1.25$ & $11.92\pm0.88$ & $0.65\pm0.01$ & $2.39$ \\
        $8$ & $70.60\pm0.21$ & $12.48\pm0.42$ & $0.66\pm0.01$ & $3.10$ \\
        \bottomrule
    \end{tabular}
\end{table}

\end{document}